%% file: main-tcs2003-arxiv.tex
\documentclass[preprint,authoryear,12pt]{elsarticle}

\usepackage{amssymb}

\usepackage{graphicx}
\usepackage{multirow}
\usepackage{algorithmic}

\input{macros.tex}

\newcommand\nbd{not~}
\newcommand\Lit{{\rm Lit}}
\newcommand\pre{{\it pre}}
\newcommand\concl{{\it concl}}
\newcommand\just{{\it just}}
\newcommand\PP{\textit{P}}
\newcommand\NL{\textit{NL}}
\newcommand\NP{\textit{NP}}
\newcommand\CONP{\textit{co-NP}}
\newcommand\SigmaP[1]{\Sigma^{P}_{#1}}
\newcommand\PiP[1]{\Pi^{P}_{#1}}
\newcommand\DeltaP[1]{\Delta^P_{#1}}
\newcommand\DP{{\rm D^P}}

\journal{arXiv.org}

\begin{document}

\begin{frontmatter}

\title{Outlier detection in default logics: the tractability/intractability frontier}


\author[dimes]{Fabrizio Angiulli}
\author[rachel]{Rachel Ben-Eliyahu--Zohary}
\author[dimes]{Luigi Palopoli}

\address[dimes]{DIMES, University of Calabria, Italy}
\address[rachel]{Software Engineering Dept., Jerusalem College of Engineering}

\begin{abstract}
In default theories, outliers denote sets of literals featuring unexpected
properties. In previous papers, we have defined outliers in default logics and
investigated their formal properties. Specifically, we have looked into the
computational complexity of outlier detection problems and proved that while
they are generally intractable, interesting tractable cases can be singled out.
Following those results,  we study here the tractability frontier in outlier
detection problems, by analyzing it with respect to $(i)$ the considered outlier
detection problem, $(ii)$ the reference default logic fragment, and $(iii)$ the
adopted notion of outlier. As for point $(i)$, we shall consider three problems
of increasing complexity, called {\em Outlier-Witness Recognition}, {\em Outlier
Recognition} and {\em Outlier Existence}, respectively. As for point $(ii)$, as
we look for conditions under which outlier detection can be done efficiently,
attention will be limited to subsets of {\em Disjunction-free propositional
default theories}. As for point $(iii)$, we shall refer to both  the notion of
outlier of \citep{ABP08} and a new and more restrictive one, called
\textit{strong} outlier.  After complexity results, we present a polynomial time
algorithm for enumerating all strong outliers of bounded size in an quasi-acyclic
normal unary default theory. Some of our tractability results rely on the {\em
Incremental Lemma} that provides conditions for a deafult logic fragment to have
a monotonic behavior. Finally, in order to show that the simple fragments of DL we deal with are still rich enough to solve interesting problems and, therefore, the tractability results that we prove are interesting not only on the mere theoretical side, insights into the expressive capabilities of these
fragments are provided, by showing that normal unary
theories express all NL queries, hereby indirectly answering a question raised
by Kautz and Selman.
\end{abstract}


\end{frontmatter}

\section{Introduction}
Consider a rational agent acquiring information about the world, where such
information is stated in the form of a sets of facts. In this setting, it is
certainly relevant
to recognize if some of these facts are  \textit{outliers}, that is, facts that
disagree with her own expectations of the
behavior of the world.
The normal behavior of the world could be encoded somehow
using one of the several non-monotonic languages
defined and studied in the Artificial Intelligence literature \citep{Win84,RussellN03,NeaJ12}. Among
those, Reiter's default logic \citep{Rei80}, has unquestionably gained a much
prominent role \citep{PooMG98,Min00}.

In the paper \citep{ABP08}, we have formally defined the notion
of outlier in the context of knowledge bases  expressed  in Reiter's default
logic
and studied some of the associated computational problems.
Outliers can be intuitively described as sets of anomalous observations in that
they feature some properties contrasting with those that can be
logically ``justified'' according to the given knowledge base.
Along with outliers, their \textit{witnesses}
are to be singled out.
Witnesses are sets of
observations that encode  the reason why the observations marked as outliers are
identified as such.

To illustrate, consider a scenario
where a credit card number is used several
times during a specific day to pay for services provided through the Internet.
This sounds normal so far, but add to that the fact that
the payment is done through different IPs, each of which is located
in a different country. It might be the case that the credit
card owner is
traveling on this particular day, but if the different countries
from which the credit card is used are located in different continents,
we might get really suspicious about who has put his hands on these credit card numbers.
In our terminology, we say that the fact that the credit card number
is used in different continents during the same day makes this credit card an outlier.

As shown in \citep{ABP08}, outlier detection problems are generally computationally
quite hard, their associated complexities ranging from $\rm D^P$-complete to $\rm D^P_3$-complete, depending
on the specific form of problem one deals with. For this reason, in
\citep{ABP10} we have looked into
significant fragments of general Reiter's default logics where
lower complexities are involved. In particular, while the outlier detection problems   remains NP-hard several language fragments, a tractable case was actually singled out. The case is
the Outlier-Witness recognition problem (the problem of telling if, for the theory at hand, a given pair of sets of literals indeed forms an outlier-witness pair),
that was proved to be solvable in polynomial time over normal unary default theories.

This initial tractability result prompted us to embark on a more systematic study of the tractability frontier associated with
outlier detection problems. 
To illustrate, in the papers \citep{ABP08,ABP10}
three basic outlier detection problems of increasing complexity
have been defined, that are:
\begin{itemize}
\item
\textit{Outlier-Witness Recognition},
\item
\textit{Outlier Recognition}, and
\item
\textit{Outlier Existence}.
\end{itemize}
Clear enough, for any of the above problems, the computational complexity underlying outlier detection depends
on the exact default language used to encode the background
knowledge base $\Delta$ and on the specific notion of outlier one intends to
deal with.
In particular, we know from previous results ) that the most general of the
problems above, namely, \textit{Outlier Existence}, and the simpler
\textit{Outlier recognition} problem,
are NP-hard even on normal unary theories.
On the other hand, on such form of theories, the problem \textit{Outlier-Witness
Recognition} is solvable in polynomial time\citep{ABP08,ABP10}.

In this paper, we continue along this line of research and try to depict the
contour of the tractability region for outlier detection problems. To this end, we shall
again consider well-known simple fragments of default logic, such as
normal unary theories, and the more
general mixed normal unary theories, and we shall
introduce a new natural notion of outlier, which we call  \textit{strong outlier}. This notion restricts  outliers to be only those featuring a tight  relationship of the outlier set with its witness.
Along the way, we shall also prove some side-, yet interesting, results
regarding the expressiveness  of the language fragments mentioned above,
and regarding some monotonicity property featured by mixed normal unary theories.

The paper is organized as follows. The next section recalls
preliminary notions about computational complexity, Reiter's
propositional default logic and its language
fragments that we will be deal with, the notion of
outliers as defined
in \citep{ABP08} and, finally, the outlier detection problems we will focus on.
Section \ref{sect:lang_properties} provides some results about the
default logic fragments referred to above. We will  first discuss the
expressive capabilities featured (via entailment) by the simple fragments of
default logic to which we have to resort in order to chart the tractability
frontier of outlier detection problems. These results show that, although very simple, these fragments are indeed rich enough to solve interesting problems. Second, we will present the Incremental Lemma,
which provides an interesting and useful monotonicity characterization of mixed normal unary
theories theories. Third, we will  prove a technical lemma (Lemma \ref{lemma:intract})
which shows how the evaluation of the truth value of a CNF formula under a
specific assignment can be accomplished
in a way that relates it to the definition of outliers.
The next four
sections are devoted to drawing  the tractability frontier associated with the outlier
detection problems
\textit{Outlier-Witness Recognition}, \textit{Outlier Recognition} and
\textit{Outlier Existence}, and
to presenting the tractable \textit{Outlier Enumeration} algorithm.
We conclude with a detailed discussion
on the transition between tractability and intractability in outlier detection using default logic in Section \ref{sect:discussion} and, in Section \ref{sect:concl}, we draw some final remarks.

\section{Preliminaries}
\label{sect:prelim}

\subsection{Complexity theory}

In this section we recall some basic definitions and results of complexity theory.
The reader is referred to
\citep{Johnson91} for additional information.

\emph{Decision} problems are maps from strings (encoding the input instance over
a fixed alphabet, e.g., the binary alphabet $\{0,1\}$) to the
set $\{ ``yes" , ``no" \}$. For a given input $x$, its size is denoted by
$||x||$.
The class $\PP$ is the set of decision problems that can be solved by a
deterministic Turing machine in time polynomial
in the input size.
machines using a work-space of logarithmic size is denoted by $\rm L$.
Throughout the paper, we shall often refer to computations carried out by
\emph{non-deterministic} Turing machines. We recall that these are Turing
machines that, at some points of the computation, may not have one single next
action to perform, but a \emph{choice} between several possible
next actions. A non-deterministic Turing machine answers a decision problem if,
on any input $x$, there is at least one sequence of choices
leading to halt in an accepting state if $x$ is a ``yes'' instance (such a
sequence is called accepting computation path); and, if $x$ is a ``no''
instance, all
possible sequences of choices lead to a rejecting state.
The class of decision problems that can be solved by non-deterministic Turing
machines using a work-space of logarithmic size is denoted by $\NL$.
The class of decision problems that can be solved by non-deterministic Turing
machines in polynomial time is denoted by $\NP$.
Problems in $\NP$ enjoy the property that any ``yes'' instance $x$ has a
\emph{certificate} for it being a ``yes'' instance, which has
polynomial length and which can be checked in polynomial time (in the size
$||x||$).
The class of problems whose complementary problems are in $\NP$ is denoted by
$\CONP$.

Recall that an \textit{oracle} is a subroutine which is supposed to require
constant computational resources to terminate.
The classes $\SigmaP{k}$, $\PiP{k}$, and $\DeltaP{k}$, forming the
\emph{polynomial hierarchy}, are defined as follows: $\SigmaP{0} = \PiP{0} =
\PP$ and for all $k\ge 1$, $\SigmaP{k}=\NP^{\Sigma^P_{k-1}}$,
$\DeltaP{k}=\PP^{\Sigma^P_{k-1}}$, and $\PiP{k}=\textit{co-}\SigmaP{k}$ where
$\textit{co-}\SigmaP{k}$ denotes the class of problems whose complementary problem
is solvable in $\SigmaP{k}$. Here, $\SigmaP{k}$ (resp.,
$\DeltaP{k}$) models computability by a non-deterministic (resp., deterministic)
polynomial-time Turing machine that may use an oracle in
$\SigmaP{k-1}$. Note that $\SigmaP{1}$ coincides with $\NP$, and that $\PiP{1}$
coincides with $\CONP$.

Next, the class $\DP^k$ is the class of problems that can be defined as a
conjunction of two problems, one from $\SigmaP{k}$ and one from $\PiP{k}$.
Thus, $\DP^k$ is a superset of both $\SigmaP{k}$ and one from $\PiP{k}$ (and, in
particular, $\DP$ is a superset of both $\NP$ and $\CONP$).

We conclude by recalling the notion of reducibility among decision problems.
A decision problem $A_1$ is \emph{polynomially reducible} to a decision problem
$A_2$, denoted by $A_1 \leq_p A_2$, if there is a
polynomial-time computable function $h$ (called reduction) such that, for every
$x$, $h(x)$ is defined and $x$ is a ``yes'' instance of $A_1$
if and only if $h(x)$ is a ``yes'' instance of $A_2$. A decision problem $A$ is
\emph{hard} for a class $\mathcal{C}$ of the polynomial
hierarchy (at any level $k\geq 1$, i.e., beyond $\PP$) if every problem in
$\mathcal{C}$ is polynomially reducible to $A$; if $A$ is hard for
$\mathcal{C}$ and belongs to $\mathcal{C}$, then $A$ is said to be
\emph{complete} for~$\mathcal{C}$. Thus, problems that are complete for
$\mathcal{C}$ are the most difficult problems in $\mathcal{C}$. In particular,
they cannot belong to some lower class in the hierarchy unless
some collapse occurs.


\subsection{Reiter's Default Logic}

Default logic was introduced by Reiter \citep{Rei80} and we next recall
basic facts about its propositional fragment. For $T$, a propositional theory, and $S$, a set of propositional formulae, $T^*$
denotes the logical closure of $T$, and $\neg S$ the set $\{\neg s \st s \in S\}$\footnote{As usual, for any letter $a$, we assume $\neg \neg a = a$.}. A set of literals $L$ is {\em inconsistent} if $\neg
{\ell} \in L$ for some literal ${\ell}\in L$. Given a literal $\ell$, $\lett{\ell}$ denotes the letter in the literal $\ell$.
Given a set of literals $L$, $\lett{L}$ denotes the set
$\{A \mid A=\lett{\ell} \mbox{ for some } \ell \in L\}$.
\subsubsection{Syntax}
A {\em propositional default theory} $\Delta$ is a pair $(D,W)$ where $W$ is a set of propositional formulae and $D$ is a set of
default rules. We assume that both sets $D$ and $W$ are finite.
A {\em default rule} $\delta$ is
\begin{equation}\label{eq:defaultrule}
\frac{\alpha : \beta_1,\ldots,\beta_m}{\gamma}
\end{equation}
where $\alpha$ (called \emph{prerequisite}), $\beta_i$, $1\le i\le
m$ (called \emph{justifications}) and $\gamma$ (called
\emph{consequent}) are propositional formulae. For $\delta$ a
default rule, $\pre(\delta)$, $\just(\delta)$, and $\concl(\delta)$
denote the prerequisite, justification, and consequent of $\delta$,
respectively. Analogously, given a set of default rules, $D =
\{\delta_1,\ldots,\delta_n\}$, $\pre(D)$, $\just(D)$, and
$\concl(D)$ denote, respectively, the sets $\{ \pre(\delta_1)$,
$\ldots$, $\pre(\delta_n) \}$, $(\just(\delta_1) \cup \ldots \cup
\just(\delta_n))$, and $\{ \concl(\delta_1), \ldots,
\concl(\delta_n) \}$. The prerequisite may be missing, whereas the
justification and the consequent are required (an empty
justification denotes the presence of the identically true literal
{\bf true} specified therein).
The informal meaning of a default rule $\delta$ is as follows: If $\pre(\delta)$ is known to hold and if it is consistent to assume
$\just(\delta)$, then infer $\concl(\delta)$.

Next, we introduce some well-known subsets of propositional
default theories relevant to our purposes
(see \citep{Rei80,KaSe91}).

{\bf Normal theories}. If the conclusion of a default rule is
identical
to the justification the rule is called {\em normal}.
A default theory containing only
normal default rules is called
{\em normal}.

{\bf Disjunction-free theories}. A default theory
$\Delta=(D,W)$ is {\it disjunction free} (DF for short)
\citep{KaSe91}, if $W$ is a set of literals, and, for each $\delta$
in $D$, $\pre(\delta)$, $\just(\delta)$, and $\concl(\delta)$ are
conjunctions of literals.

{\bf Normal mixed unary theories}. A DF default theory is {\em
normal mixed unary} (NMU for short) if its set of defaults contains
only rules of the form $\frac{\alpha:\beta}{\beta}$, where $\alpha$
is either empty or a single literal and $\beta$ is a single literal.

{\bf Normal and dual normal unary theories}. An NMU default
theory is {\em normal unary} (NU for short) if the prerequisite of
each default is either empty or positive. An NMU default theory is
{\em dual normal} (DNU for short) unary if the prerequisite of each
default is either empty or negative.

Next, the concept of quasi-acyclic theory is introduced. We begin by introducing the notions of atomic dependency graph and that of
tightness of a NMU default theory.
\begin{definition}[Atomic Dependency Graph]\rm
Let $\Delta = (D,W)$ be a NMU default theory. The {\em atomic dependency graph} $(V,E)$ of $\Delta$ is a directed graph such
that
\begin{itemize}
\item[--]
$V = \{ l \mid l \mbox{~is~a~letter~occurring~in~} \Delta \}$, and
\item[--]
$E = \{ (x,y) \mid $ letters $x$ and $y$ occur respectively in
the prerequisite and the consequent of a default in $D \}$.
\end{itemize}
\end{definition}
\begin{definition}[A set influences a literal]
Let $\Delta = (D,W)$ be an NMU default theory. We say that a set of literals $S$ \emph{influences} a literal $l$ in $\Delta$ if for some $t \in S$
there is a path from $\lett{t}$ to $ \lett{l}$ in the atomic dependency graph of $\Delta$.
\end{definition}
\begin{definition}[Tightness of an NMU theory]\rm
The {\em tightness} $c$ of an NMU default theory is the size $c$ (in terms of number of atoms) of
the largest strongly connected component (SCC) of its atomic dependency graph.
\end{definition}
Intuitively, a quasi-acyclic NMU default theory is a theory whose tightness is upper-bouded by a fixed constant.
\begin{definition}[quasi-acyclic NMU theory]\rm
Given a fixed positive integer $c$, a NMU default theory is said to be
($c$-)\textit{quasi-acyclic}, if its tightness is not greater than $c$.
\end{definition}
Figure \ref{fig:classes_map} highlights the set-subset relationships holding between
the above defined fragments of default logic.
\begin{figure}
\begin{center}
\includegraphics[width=250px,height=150px]{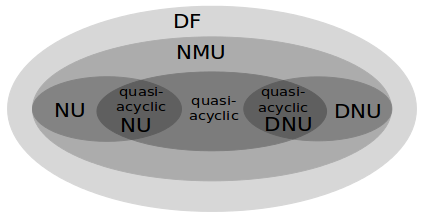}
\end{center}
\caption{A map of the investigated default theory fragments}
\label{fig:classes_map}
\end{figure}
\subsubsection{Semantics}
The formal semantics of a default theory $\Delta$ is defined in terms of {\em extensions}. A set ${\cal E}$ is an extension for a theory $\Delta = (D,W)$
if it satisfies the following set of equations:
\begin{itemize}
\item
$E_0=W$,
\item
for $i\ge 0$, $E_{i+1} = E_i^\ast \cup \left\{ \gamma \mid
\frac{\alpha:\beta_1,\ldots,\beta_m}{\gamma} \in D, \alpha \in E_i,
\neg\beta_1\not\in{\cal E}, \ldots, \neg\beta_m\not\in{\cal E}
\right\}$,
\item
$\displaystyle {\cal E} = \bigcup_{i=0}^\infty E_i$.
\end{itemize}
Given a default $\delta$ and an extension $\cal E$, we say that $\delta$ is applicable in $\cal E$ if $\pre(\delta) \in \cal E$
and $(\not \exists c \in \just(\delta ))(\neg c \in \cal E)$.

It is well known that an extension ${\cal E}$ of a finite propositional default theory
$\Delta=(D,W)$ can be finitely characterized through the set
$D_{\cal E}$ of the {\em generating defaults} for ${\cal E}$ w.r.t.
$\Delta$ (the reader is referred to \citep{Rei80,ZhMa90} for definitions).


A finite propositional default theory $\Delta=(D,W)$ has an extension
$\cal E$ iff there exists a set $D_{\cal E} \subseteq D$, the
generating defaults of $\cal E$ w.r.t. $\Delta$, that can be
partitioned into a finite number of strata $D^{(0)}_{\cal
E},D^{(1)}_{\cal E},\ldots,D^{(n)}_{\cal E}$, such that:
\begin{itemize}
\item
$D^{(0)}_{\cal E} = \{ \delta \mid \delta\in D_{\cal E},
\pre(\delta)\in W^\ast \}$,
\item
for each $i$, $1\le i\le n$, $D^{(i)}_{\cal E} = \{ \delta \mid
\delta\in D_{\cal E} - \bigcup_{j=0}^{i-1} D^{(j)}_{\cal E},
\pre(\delta)\in (W \cup \concl(\bigcup_{j=0}^{i-1}D^{(j)}_{\cal
E}))^\ast \}$,
\end{itemize}
and, moreover:
\begin{itemize}
\item
$(\forall\delta\in D_{\cal
E})(\forall\beta\in\just(\delta))(\neg\beta\not\in(W\cup
\concl(D_{\cal E}))^\ast )$, and
\item
$(\forall\delta\in D)(\pre(\delta)\in(W\cup \concl(D_{\cal E}))^\ast
\wedge (\forall\beta\in\just(\delta))(\neg\beta\not\in(W\cup
\concl(D_{\cal E}))^\ast \Rightarrow \delta\in D_{\cal E})$.
\end{itemize}
If such a set $D_{\cal E}$ exists, then ${\cal E} = (W \cup \concl(D_{\cal E}))^\ast$ is an extension of $\Delta$.
Next we introduce a characterization of an extension of a finite DF propositional theory
which is based on a lemma from \citep{KaSe91}.
\begin{lemma}\label{ks1}
Let $\Delta = (D,W)$ be a DF default theory; then $\cal E$ is an
extension of $\Delta$ if and only if there exists a sequence of defaults
$\delta_1,...,\delta_n$ from $D$ and a sequence of sets
$E_0,E
_1,...,E_n$, such that for all $i > 0$:
\begin{itemize}
\item $E_0=W$,
\item $E_i=E_{i-1} \cup \concl(\delta_i)$,
\item $\pre(\delta_i) \in E_{i-1}$,
\item $(\not \exists c \in \just(\delta_i))(\neg c \in E_n)$,
\item $(\not \exists \delta \in D) (\pre(\delta)\in E_n \wedge
\concl(\delta) \not\subseteq E_n \wedge (\not\exists c \in
\just(\delta))(\neg c \in E_n))$,
\item $\cal E$ is the logical closure of $E_n$,
\end{itemize}
where $E_n$ is called the {\em signature set} of $\cal E$ and is
denoted $\liter{{\cal E}}$ and the sequence of rules
$\delta_1,...,\delta_n$ is the set $D_{\cal E}$ of generating
defaults of $\cal E$.
\end{lemma}

Although default theories are {\em non-monotonic}, normal default
theories satisfy the property of {\em semi-monotonicity} (see
Theorem 3.2 of \citep{Rei80}).

Semi-monotonicity in default logic means the following: Let
$\Delta=(D,W)$ and $\Delta'=(D',W)$ be two default theories such
that $D\subseteq D'$; then for every extension ${\cal E}$ of $\Delta$
there is an extension ${\cal E}'$ of $\Delta'$ such that ${\cal E} \subseteq
{\cal E}'$.

A default theory may not have any extensions (an example is the theory
$(\{\frac{:\beta}{\neg \beta}\},\emptyset)$. A default theory
is called {\em coherent} if it has at least one extension, and
incoherent otherwise. Normal default theories are always coherent. A
coherent default theory $\Delta = (D,W)$ is called {\em
inconsistent} if it has just one extension which is inconsistent. By
Theorem 2.2 of \citep{Rei80}, the theory $\Delta$ is inconsistent iff
$W$ is inconsistent.

The theories examined in this paper are always coherent and consistent,
since only normal default theories $(D,W)$ with
$W$ a consistent set of literals are taken into account.

The {\em entailment problem} for default theories is as follows:
Given a default
theory $\Delta$ and a propositional formula $\phi$, does every
extension of $\Delta$ contain $\phi$? In the affirmative case, we
write $\Delta \models \phi$. For a set of propositional formulas
$S$, we analogously write $\Delta \models S$ to denote $(\forall
\phi \in S)(\Delta \models \phi)$.

\subsection{Outliers in Default Logic}
\label{sect:outlier}

The issue of outlier detection in default theories has been extensively
discussed in \citep{ABP08}. The formal definition of outlier introduced in that paper is as explained next.
For a given set $W$ and a collection of sets
$S_1,\ldots,S_n$, $W_{S_1,\ldots,S_n}$ denotes the set $W \setminus
(S_1 \cup S_2 \cup \ldots \cup S_n)$.
\begin{definition} [Outlier and Outlier Witness Set]\rm\citep{ABP08}
\label{outlierD} Let $\Delta=(D,W)$ be a propositional default
theory and let $L \subseteq W$ be a set of literals.
If there exists a non-empty subset $S$ of $W_L$ such that:
\begin{enumerate}
\item $(D,W_S) \models \neg S$, and \item $(D,W_{S,L})
\not\models \neg S$
\end{enumerate}
then $L$ is an {\em outlier} set in $\Delta$ and $S$ is an {\em outlier
witness} set for $L $ in $\Delta$.
\end{definition}

The intuitive explanation of the different roles played by an outlier and
its witness is as follows. Condition ($i$) of Definition
\ref{outlierD} states that the outlier witness set $S$ denotes something
that does not agree with the knowledge encoded in the defaults.
Indeed, by removing $S$ from the
theory at hand, we obtain $\neg S$. In other words, if $S$ had not been explicitly
observed, then, according to the given defaults, we would have concluded the exact
opposite. Moreover, condition ($ii$) of Definition
\ref{outlierD} states that the outlier $L$ is a set of literals that, when
removed from the theory,
makes such a disagreement disappear.
Indeed, by removing both $S$ and $L$ from the theory, $\neg S$ is no
longer obtained. In other words, disagreement for $S$ is a
consequence of the presence of $L$ in the theory.
To summarize, the set $S$ witnesses that the piece of
knowledge denoted by $L$ behaves, in a sense, exceptionally,
thus telling us that $L$ is an outlier
set and $S$ is its associated outlier witness set.

The above intuition is illustrated by referring to the example
on stolen credit card numbers given in the introduction. A default theory
$\Delta = (D,W)$ that encodes that episode might be as follows:
\begin{itemize}
\item[--]
$D = \left\{ \frac{CreditNumber : \neg MultipleIPs}{\neg MultipleIPs}
\right\}$,
\item[--]
$W = \{CreditNumber, MultipleIPs\}$.
\end{itemize}
where, intuitively: (a) {\em CreditNumber} is true if a given credit card number has been used for purchasing something today; (b) {\em MultipleIPs} is true if that credit card number has been used from (computers located in) different continents today. Here, the credit card number might be stolen, for
otherwise it would not have been used over different continents during the same day. Accordingly, $L=\{CreditNumber\}$ is an outlier set here, and $S=\{MultipleIPs\}$ is the associated witness
set. This agrees with our intuition that an outlier is, in some sense, abnormal and that the corresponding witness testifies to it.
Note that sets of outliers and their corresponding witness sets are selected among those explicitly embodied in the given knowledge base. Hence, we look at outlier detection using default
reasoning essentially as a knowledge discovery technique, whereby abnormalities and their presumable associated explanations can be automatically singled out in the knowledge base at hand.
Several application examples showing the usefulness of the notion of outlier in default logics are given in
\citep{ABP08}.

\subsection{Outlier detection problems}

The intrinsic complexity of reasoning in default logics has been studied in some classical papers 
\citep{KaSe91,Stillman92,Gottlob92}.
The notion of outlier is based on entailment, but its structure makes the associated computational problems usually harder than basic entailment ones, as demonstrated in \citep{ABP08,ABP10}, where the complexity of discovering outliers in default theories under various classes of default logics has been investigated. The main recognition tasks in outlier detection
are the \textit{Outlier Existence}, \textit{Outlier Recognition} and the \textit{Outlier-Witness Recognition}
problems (also called $Outlier$, $Outlier(L)$ and $Outlier(S)(L)$, respectively,
in \citep{ABP08}), and are defined as follows:
\begin{itemize}
\item[-]
\textbf{Outlier-Witness Recognition Problem} ($Outlier(L)(S)$):
\textit{Given a default theory $\Delta=(D,W)$ and two sets of literals $L \subset W$ and $S \subseteq W_L$, is $L$ an outlier set
with witness set $S$ in $\Delta$?}
\item[-]
\textbf{Outlier Recognition Problem} ($Outlier(L)$):
\textit{Given a default theory $\Delta=(D,W)$ and a set of
literals $L \subseteq W$, is $L $ an outlier set in $\Delta$ (for some witness set $S$)?}
\item[-]
\textbf{Outlier Existence Problem} ($Outlier$):
\textit{Given a default theory $\Delta=(D,W)$ and a positive integer $k$,
is there any outlier set $L$ in $\Delta$ such that $|L|\le k$ (for some witness set $S$)? }
\end{itemize}
Table \ref{table:results} summarizes previous complexity results,
together with the results that constitute the contributions
of the present work that will be detailed later in this section.
\begin{table}[t]%
\footnotesize
\begin{center}
\caption{Complexity results for outlier detection ($^{*}$=reported in \citep{ABP08}, $^{**}$=reported in \citep{ABP10})\label{table:results}}{%
\begin{tabular}{|c|c||c|c|c|c|}
\hline
\multirow{3}{*}{\it Problem} & \multirow{3}{*}{\begin{tabular}{c}\it Outlier\\ \it Type\end{tabular}} &
\multirow{3}{*}{\begin{tabular}{c}\it General\\ \it Default\end{tabular}} &
\multirow{3}{*}{\begin{tabular}{c} DF\\ \it Default\end{tabular}} &
\multirow{3}{*}{\begin{tabular}{c} (D)NU\\ \it Default\end{tabular}} &
\textit{quasi-acyclic} \\
& & & & & (D)NU \\
& & & & & \textit{Default} \\
\hline\hline
\multirow{4}{*}{\it Outlier Existence} & \multirow{2}{*}{\it General} & $\rm\Sigma^P_3$-c & $\rm\Sigma^P_2$-c & NP-c & {NP-c} \\
& & Th. 4.1$^{*}$ & Th. 4.1$^{*}$ & Th. 3.5$^{**}$ & Th. \ref{th:strong_outlier_existence} \\
\cline{2-6}
& \multirow{2}{*}{\textit{Strong}} & \multicolumn{2}{|c|}{NP-hard} & {NP-c} & {NP-c} \\
& & \multicolumn{2}{|c|}{Th. \ref{th:out_cyc_exist}} & Th. \ref{th:out_cyc_exist} & Th. \ref{th:strong_outlier_existence} \\
\hline
\multirow{4}{*}{\it Outlier Recognition} & \multirow{2}{*}{\it General} & $\rm\Sigma^P_3$-c & $\rm\Sigma^P_2$-c & NP-c & \underline{NP-c} \\
& & Th. 4.3$^{*}$ & Th. 4.3$^{*}$ & Th. 3.6$^{**}$ & Th. \ref{th:out_acy_rec} \\
\cline{2-6}
& \multirow{2}{*}{\textit{Strong}} & \multicolumn{2}{|c|}{NP-hard} & \underline{NP-c} & \underline{P} \\
& & \multicolumn{2}{|c|}{Th. \ref{th:strout_cyc_rec}} & Th. \ref{th:strout_cyc_rec} & Th. \ref{th:strout_acy_rec} \\
\hline
& \multirow{2}{*}{\it General} & $\rm D^P_2$-c & $\rm D^P$-c & P & P \\
\it Outlier-Witness & & Th. 4.6$^{*}$ & Th. 4.6$^{*}$ & Th. 3.1$^{**}$ & Th. 3.1$^{**}$ \\
\cline{2-6}
\it Recognition & \multirow{2}{*}{\textit{Strong}} & \multicolumn{2}{|c|}{NP-hard } & P & P \\
& & \multicolumn{2}{|c|}{Lemma \ref{th:out_wit_rec_df}} & Lemma \ref{th:out_wit_rec} & Lemma \ref{th:out_wit_rec} \\
\hline
\end{tabular}}
\end{center}
\end{table}

In particular, the complexity of outlier detection tasks has been studied in
\citep{ABP08} for general
and disjunction-free (DF) default theories and in \citep{ABP10} for (dual) normal
unary default theories.
The results reported in the cited papers pointed out that
the general problem of recognizing an outlier set
is always intractable (Theorem 4.3 in \citep{ABP08} and Theorem 3.6 in
\citep{ABP10}).
As for recognizing an outlier together with its witness, this problem is
intractable for general and disjunction-free default theories (Theorem 4.6 in
\citep{ABP08}), but can be solved in polynomial time if NU (as well as DNU)
default theories are considered \citep{ABP10}.
Regarding the latter result,
it is worth recalling that, while for both NU and DNU default theories
the entailment of a literal can be decided in polynomial time, deciding the
entailment
in DF default theories is intractable \citep{KaSe91}.

\section{Computational and expressive properties of NMU and NU default theories}
\label{sect:lang_properties}

In this section we investigate some important properties of NMU and NU theories. Before that, it is worth pointing out that the results we shall prove to hold for NU (DNU, resp.) theories immediately apply to DNU (NU, resp.) theories, since given an NU (DNU, resp.) theory $\Delta$, the dual theory $\overline{\Delta}$ of $\Delta$ is obtained from
$\Delta$ by replacing each literal $\ell$ in $\Delta$ with $\neg \ell$ is a DNU (NU, resp.) theory that has the same properties of its dual.

Specifically, we shall first assess
to which extent the admittedly quite simple NMU and NU default theories
are powerful enough to solve (via entailment)
interesting computational problems,
hereby indirectly answering a
long-lasting questions posed by Kautz and Selman in \citep{KaSe91}.
We will do so
by first showing that the entailment problem on (quasi-acyclic) NMU theories is
\CONP-complete (Section \ref{sect:nmu_entail}), and then proving that
NU default theories (and, thereby, also NMU theories) are powerful enough to
express all decision problems included in $\NL$ (Section \ref{sect:nu_expr}).

Next, Section \ref{sect:incr} will present the \textit{Incremental Lemma},
which provides an interesting monotonicity characterization in NMU theories.

Then, in Section \ref{sect:technical} we will prove a technical lemma which
demonstrates how the evaluation of the truth value of a CNF formula under
a specific assignment can be accomplished in a way that relates it to the
definition of outliers.

Results presented here will be exploited in the rest of the paper
in order to characterize the complexity of outlier detection problems in default
theories.

\subsection{The complexity of the entailment problem for NMU theories}\label{sect:nmu_entail}

\begin{theorem}\label{th:acy_nmu_entail}
Let $\Delta$ be a NMU propositional default theory and let $l$ be a literal.
Then, the problem of telling if $\Delta \models l$ is co-NP-complete.
\end{theorem}
\begin{proof}
(Membership) Membership in co-NP follows immediately from membership in co-NP of the entailment problem for disjunction-free propositional
default theories \citep{KaSe91}.

(Hardness) Let $\Phi$ be a boolean formula in 3CNF on the set of variables
$X=x_1,\ldots,x_n$, such that $\Phi = C_1 \wedge \ldots \wedge
C_m$, with $C_k = t_{k,1} \vee t_{k,2} \vee t_{k,3}$, and each
$t_{k,1},t_{k,2},t_{k,3}$ is a literal on the set $X$, for
$k=1,\ldots,m$. The default theory $\Delta(\Phi) =
(D(\Phi),\emptyset)$ is associated with $\Phi$, where
$D(\Phi)$ is $D_1 \cup D_2 \cup D_3$, with:
\begin{eqnarray*}
D_1 & = & \left\{ \frac{:x_i}{x_i}, \frac{:\neg
x_i}{\neg x_i} \mid i=1,\ldots,n
\right\}, \\
D_2 & = & \left\{ \frac{t_{k,j} : c_k}{c_k} \mid
k=1,\ldots,m;
j=1,2,3 \right\}, \mbox{ and} \\
D_3 & = & \left\{ \frac{:\neg c_k}{\neg c_k}, \frac{\neg
c_k:l}{l} \mid k=1,\ldots,m
\right\},
\end{eqnarray*}
where $l$ is a new letter distinct from those
occurring in $\Phi$. It is shown next that $\Phi$ is unsatisfiable
iff $\Delta(\Phi) \models l$.
We recall that
the unsatisfiability problem of a 3CNF is a well-known co-NP-complete problem.
Consider a generic extension $\cal E$ of $\Delta(\Phi)$. From the rules
in the set $D_1$, $\cal E$ is such that for each $i=1,\ldots,n$,
either $x_i\in E$ or $\neg x_i\in E$.

(Only If part) Suppose that $\Phi$ is unsatisfiable. Then, for
each truth assignment $T$ on the set of variables $X$, there exists
at least a clause, say $C_{f(T)}$, $1\le f(T)\le m$, that is not
satisfied by $T$. Because of the rules in the set $D_2$,
$c_{f({\cal T}_{E \cap (X \cup \neg X)})} \not\in {\cal E}$, and
from rules in the set $D_3$, $\neg c_{f({\cal T}_{{\cal E} \cap (X \cup
\neg X)})} \in {\cal E}$ and $l\in {\cal E}$.

(If part) Suppose that $\Delta(\Phi) \models l$. Then, for each extension
$\cal E$ of $\Delta(\Phi)$, there exists $g({\cal E})$, $1\le g({\cal E})\le m$, such
that $\neg c_{g({\cal E})}\in {\cal E}$. For each truth assignment $T$ on the
set of variables $X$, let $E(T)$ denote the set containing
all the extensions $\cal E$ of $\Delta(\Phi)$ such that ${\cal E} \supseteq
Lit(T)$. Then, for each ${\cal E} \in {E}(T)$, $\neg c_{g({\cal E})} \in
{\cal E}$ implies that none of the rules in the set $D_2$ having $c_{g({\cal E})}$
as their conclusion belong to the set of generating defaults of
$\cal E$. Thus, the clause $C_{g({\cal E})}$ is not satisfied by $T$. As this
holds for each truth assignment $T$,
$\Phi$ is unsatisfiable.
\end{proof}

From the above result, the complexity
of the entailment problem for propositional quasi-acyclic NMU default theories can be promptly derived:

\begin{corollary}\label{th:nmu_entail}
Let $\Delta$ be a quasi-acyclic NMU propositional default theory and let $l$ be a literal.
Then, the problem of telling if $\Delta \models l$ is co-NP-complete.
\end{corollary}
\begin{proof}
The statement immediately follows by noting that the theory employed in the reduction of
Theorem \ref{th:nmu_entail} has tightness $1$.
\end{proof}

For the sake of completeness, before closing this section, the cost of the entailment problem for NU theories is recalled next.

\begin{proposition} {\rm (proved in
\citep{KaSe91,Ben02})}\label{th:nuentail}
\label{tracKS}\label{tracBen} Let $\Delta$ be a NU propositional default theory and let $L$ be a set of literals.
Deciding if $\Delta\models L$ can be done in time ${\cal O}(n^2)$, where $n$ is the size of the theory $\Delta$.
\end{proposition}

\subsection{The Expressive Power of Normal Unary Default Theories}\label{sect:nu_expr}

This section is devoted to show that, in fact, the yet very simple default theories we mainly refer to in this paper are actually rich enough to express interesting knowledge.

To this end,
we follow the literature (see \citep{DantsinEGV01} and informally
define the {\em expressive power} of a (logic) language as the set of
properties over finite structures that language allows to express.
We assume an arbitrarily large but finite domain of constants $\cal U$ to be given
over which input finite structures
are defined \citep{CadoliEG97}. Here, the input finite
structure $\cal S$ will be assumed to be encoded into a fixed-schema
{\em completed} relational database\footnote{We assume here that the reader is
familiar with basic notions regarding relational database theory. See, e.g.,
\citep{AbiteboulHV95} for an excellent source of material on
this topic.} $db_{\cal S}$ and, accordingly, the expression of the language
in question expressing a given property will be regarded as a query to be
evaluated against such an input database.
By {\em completed} relational database we mean a database $db$ for which the following holds:
for each relation $R$ in $db$ there is a relation in $db$, we shall call it complement of $R$ and denote it by $\overline{R}$, with the same schema as $R$, such that $\overline{R} = \{{\bf
t} \in {\cal U} \times \ldots \times {\cal U} \mid {\bf t} \not\in R\}$. In other words, in a completed database, we will find, together with each relation, also its complement with respect to
active tuple domain ${\cal U} \times \ldots \times {\cal U}$.
From here one, we shall identify
the finite structure $\cal S$ with its completed relational encoding $db_{\cal S}$.
%

We shall say that the database $db_{\cal S}$ is {\em ordered} if its schema contains
relations $First$, $Succ$ and $Last$ with the meaning of encoding a
complete ordering of all domain elements included in $\cal U$.

A (boolean) query $q$ expressed in the given language defines a {\em
generic mapping} $m_q$ that associates to each input database $db_{\cal S}$ (over the given fixed schema) a boolean value
$m_q(db_{\cal S}) \in \{true, false\}$\footnote{We recall that genericity
means invariance under domain isomorphisms. Moreover, we recall that queries
are usually and more generally defined as mappings from databases to
databases. Here, we stick to Boolean queries as these are sufficient for
foregoing presentation.}.

All that given, by the {\em expressive power} of a
language $\cal L$ we mean the set of mappings $m_q$ for all queries $q$
expressible in $\cal L$ by some query expression $E$ (by abusing notation,
this expression is usually identified with the query $q$ it defines and
therefore called a query itself).

Next, we relate NU default theories to queries, as follows. A {\em NU query
form} $qf$ is a pair $\langle \Delta_{qf}, a_{qf} \rangle$, where $a_{qf}$
is a positive literal and $\Delta_{qf}$ is a NU default theory of the form
$(\emptyset, D_{qf})$, for some set of NU defaults $D_{qf}$.
A query form
defines a query over relational databases, as follows. For each database
$db$, let $W_{db}$ be the set of (positive) literals
naturally encoding $db$ in logical form as $W_{db} = \{ R({\bf a}) \mid \mbox{$R$
is a relation in $db$ and } {\bf a} \in R\}$. Then, for each database $db$,
the query form $qf$ defines the query $q_{qf}$ such that $q_{qf} (db) =
true$ if and only if $(W_{db},D_{qf}) \entails a_{qf}$.

Now that the mechanism relating default theories and queries over finite
structures has been established, we can discuss the
expressive power of NU default theories.

In order to asses the expressive power of the languages we are interested in,
we shall refer to descriptive complexity theory \citep{Ebbin95} according to which this is
measured by relating it to complexity classes.
Informally speaking, as
reported in \citep{DantsinEGV01}, in order to prove that a language
expresses a complexity class $C$ one has to show that every $C$-machine
working on a (proper encoding of a)   finite structure can be represented by
an expression of the language. An alternative way to go is to show that the
language at hand is equivalent to (or subsumes) another language whose expressive power
has already been established.  We will use the latter strategy.

To this end, we consider the language $datalog^+$,
that is to say the language $datalog$ augmented with
the possibility of specifying negated extensional
predicates in rule bodies \citep{AbiteboulHV95} 
and, specifically, the sublanguage $datalog^+(1)$, that is, the fragment of $datalog^+$ where
each rule has at most one intentional  predicate in the body.

The following result is an immediate consequence of Theorem 8 in \citep{DantsinEGV01}.

\begin{theorem} {\rm(cf. Theorem 8 in \citep{DantsinEGV01})}
The language $datalog^+(1)$ captures NL over finite
structures encoded as completed relational databases on ordered domains.
\end{theorem}

Next, we show that the language of NU default theories
is at least as expressive as $datalog^+(1)$ and, as such, expresses all NL queries over completed ordered databases.

\noindent
{\bf Notation.} In the following, in order to make the presentation more easily followed, we shall often use theories where defaults contain variables: these are to be inteded simply a shorthands for the set of ground (that is, propositional) defaults obtained from those by substituting in all consistent ways variables with constants appearing in the theory.

\begin{theorem}\label{th:nu_expr}
On completed ordered databases, NU default theories are at least as expressive as $datalog^+(1)$.
\end{theorem}
\begin{proof}
In order to prove our result, we will show that for
each $datalog^+(1)$ program $P$ defining,
together with the letter $g$, the query $q_P$ over ordered relational databases
(i.e., for each database $db$, $q_P(db) = true$ iff $P \cup W_{db} \entails g$),
there exists a NU
query form $qf$ such that $q_{qf}$ coincides with $q_P$.

While sticking with propositional defaults only, for the sake of the
presentation, we shall use below variables in defaults: wherever variables
appear, a default will be intended as the mere representation of the set of
propositional defaults obtained by substituting in any possible consistent
ways those variables with elements from $\cal U$.
Moreover, references to single atom containing variables
have to be intended to represent a certain propositional atom
obtained by substituting variables therein appearing
with elements from $\cal U$.

Let us illustrate next the translation of a generic rule included in $P$ into a set
of defaults. Let $r$ be the generic rule of the following form:
\[ r \equiv p({\bf Y}) \leftarrow q({\bf X}_0),
a_1({\bf X}_1), \ldots , a_m({\bf X}_m),\nbd a_{m+1}({\bf X}_{m+1}), \ldots, \nbd a_n({\bf X}_n), ~(0\le m\le n) \]
where $q$ is an intensional (i.e., non-database) predicate,
and the $a_i$ denote extensional (i.e., database) predicates
(all the predicates occurring in the body of the rule are optional).

The rule $r$ is translated into the set of defaults $D_r$,
consisting of the following default rules:
\begin{itemize}
\item
$\delta^r_i \equiv \displaystyle \frac{\overline{a}_i({\bf X}_i):
\neg p_r({\bf Z})}
{\neg p_r({\bf Z})}$
~$(1 \leq i \leq m)$,
\item
$\delta^r_j \equiv \displaystyle \frac{a_j({\bf X}_j):
\neg p_r({\bf Z})}
{\neg p_r({\bf Z})}$
~$(m+1 \leq j \leq n)$,
\item
$\delta^r_0 \equiv \displaystyle \frac{q({\bf X}_0):
p_r({\bf Z})}
{p_r({\bf Z})}$, ~and
\item
$\delta^r \equiv \displaystyle \frac{p_r({\bf Z}): p({\bf Y})}{p({\bf Y})}$.
%
%
\end{itemize}
where $\bf Z$ denotes the set of variables ${\bf X}_0\cup \ldots\cup {\bf X}_n$
occurring in the rule $r$, among which an arbitrary ordered is assumed,
$\overline{a}_i$ denotes the complemental letter of $a_i$
and $p_r$ is a novel predicate name associated with the rule $r$.

Notice that the same translation schema as above applies to recursive rules possibly occurring in $P$.
Accordingly, the $datalog^+(1)$ program $P$ is translated into the set of defaults $D_P$ obtained by translating all its rules into defaults.

We claim that the NU query form $qf = \langle (D_P, \emptyset), g \rangle$ is such that, for any finite
database $db$, $q_P(db) = q_{qf}(db)$. In other
words, we have to show that $g$ belongs to the minimal model of the program
resulting by adding $W_{db}$ to $P$ if and only if $g$ is a
cautious consequence of the NU default theory $\Delta_P = (W_{db},D_P)$.

In what follows, we will refer to the defaults $\delta^r_i$ with $i>0$, as
\textit{extensional defaults} (since an extensional
predicate occurs in their prerequisite), and to the remaining ones,
that is the defaults $\delta^r_0$ and $\delta^r$, as
\textit{intensional defaults} (since only intensional
predicates occur in these rules).

Given a generic extension $E$ of $\Delta_P$, let $E_P$
denote the subset of $E$ consisting of the positive literals
whose extensional predicate occurs in the program $P$.

In what follows, for simplicity,
when we speak of the minimal model of
$P \cup W_{db_{\cal S}}$,
we will take into account only intensional predicates,
since the extensional component of the minimal model of $P$
is always given by $W_{db_{\cal S}}$.

Consider the set ${\cal E}^{\bot}$ of extensions $E$ of the default theory $\Delta_P$, having the property that
there is no intensional default preceding an extensional one in the sequence $D_E$ of generating defaults of $E$.
\begin{claim}\label{clB}
The set ${\cal E}^\bot$ contains a unique extension.
\end{claim}
\begin{proof}
Consider a generic extension $E$ of ${\cal E}^\bot$.
First of all, notice that extensional defaults contain in their prerequisite
only extensional predicates and in their conclusion only literals of the form
$\neg p_r({\bf Z})$.

Since extensional defaults precede intensional ones in the sequence of
generating defaults $D_E$ associated with $E$,
it follows that $E$ contains all the literals the form $\neg p_r({\bf Z})$
occurring in at least one extensional default of $\Delta_P$ such that the
extensional atom in the
corresponding prerequisite is true, say this set $Neg$. Due to the maximality
property of extensions, $Neg$ is the maximal set of negative literals
that can be deduced by means of the defaults in $\Delta_P$.

As for the positive literals occurring in $E$,
they will be the maximal set of positive literals that can be inferred by means
of the defaults in $\Delta_P$
after that all the literals of the form $\neg p_r({\bf Z})$ in $Neg$ are assumed
to be part of the extension.
Say this set of positive literals $Pos$.

Thus, all the sequences of generating defaults such that extensional defaults
precede intensional ones
lead to the same sets $Neg$ and $Pos$ and, hence, to the same extension.
In other words, these sequences are equivalent, in terms of the extension that
they produce,
up to permutations that preserve the property that extensional defaults precede
intensional ones.

Thus, it can be concluded that all the extensions in ${\cal E}^\bot$
must coincide, and ${\cal E}^\bot$ contains one single extension. 
\end{proof}

In the following, with a little abuse of notation, we will denote by ${\cal E}^{\bot}$ the only
extension belonging to ${\cal E}^{\bot}$.

The following result states that ${\cal E}^\bot_P$ represents the set of all and only the intensional atoms
entailed by the theory $\Delta_P$.

\begin{claim}
Let $g$ be an extensional atom. Then, $\Delta_P \entails g$ if and only if $g \in {\cal E}^\bot_P$.
\end{claim}
\begin{proof}
From what stated in Claim \ref{clB}, since $Neg$ is the maximal set of negative literals that
can be inferred from $\Delta_P$, it follows that the set $Pos$, consisting of the positive
intensional literals contained in ${\cal E}^\bot$, is contained in every extension of $\Delta_P$. 
\end{proof}

It remains to show that the minimal model of $P\cup W_{db_{\cal S}}$
coincides with ${\cal E}^\bot_P$. We preliminarily need to prove a further
technical result.

Let $P$ be a $datalog(1)$ program over a ordered completed database $db$. Consider the program $P^{db}$ obtained from $P$
by (1) deleting all the rules whose body is falsified by at least one extensional atom in $db$, and
(2) removing all the occurrences of extensional atoms in the remaining rules.

\begin{claim}
The minimal model ${\rm MM}(P\cup db)$ of $P\cup db$ coincides with ${\rm MM}(P^{db})\cup db$, where ${\rm MM}(P^{db})$ denotes
the minimal model of $P^{db}$.
\end{claim}
\begin{proof}
Let $I'$ be a generic interpretation of $P^{db}$. The interpretation $I=I'\cup db$ is such that
the rules whose body is falsified by an extensional atom in $db$ are true in $I$.

As for the remaining rules, consider a generic interpretation
$I$ of $P\cup db$ and a rule $r$ of $P$.
Let $r'$ denote the rule $r$ obtained from $r$
by removing all the occurrences of extensional atoms
therein appearing.
It can be verified that $r$ is true in $I$
if and only if $r'$ is true in $I' = I\setminus db$.

Thus, it can be concluded that there is a one-to-one correspondence between
models of $M$ of $P\cup db$ and models $M'=M\setminus db$
of $P^{db}$,
from which the theorem statement immediately follows. 
\end{proof}

Now, notice that by definition of ${\cal E}^\bot$,
for any rule $r$ falsified
by an extensional atom,
the literal $\neg p_r({\bf Z})$ belongs to
${\cal E}^\bot$.
Hence, it follows that $Pos$, the set of all and only the positive literals
occurring in ${\cal E}^\bot$,
coincides with the unique extension
of the default theory $(D_P^{db},\emptyset)$,
with $D_P^{db}$ consisting of
the intensional default rules
associated with rules $r$ belonging to $P^{db}$.
Moreover, since the positive literals $p_r({\bf Z})$ are
distinguished atoms, each associated with a different rule $r$,
the set of defaults $D_P^{db}$ is equivalent in its turn
to the set $D_P^{db'}$ obtained from $D_P^{db}$
by substituting each pair of defaults of the form
$\frac{q({\bf X}_0):p_r({\bf Z})}{p_r({\bf Z})}$,
$\frac{p_r({\bf Z}):p({\bf Y})}{p({\bf Y})}$
with the default
$\frac{q({\bf X}_0):p({\bf Y})}{p({\bf Y})}$.

The last step consists in showing that the unique extension
of $(D_P^{db'},\emptyset)$ coincides with
the minimal model of $P^{db}$.
Since no negative literal occur in $(D_P^{db'},\emptyset)$,
then its extension must coincide with
the minimal model of the logic program $P'$
obtained by mapping each default
$\frac{q({\bf X}_0):p({\bf Y})}{p({\bf Y})}$
in $D_P^{db'}$ to the rule
$p({\bf Y}) \leftarrow q({\bf X}_0)$.
The result follows,
since the program $P'$ is exactly $P^{db}$.
Thus, the minimal model of
$P\cup W_{db_{\cal S}}$
coincides with ${\cal E}^\bot_P$,
and this concludes the proof. 
\end{proof}


To provide an example of the translation exploited in Theorem
\ref{th:nu_expr}, we next consider the \textit{reachability problem}:
given a graph $G$, encoded in a database $db_G$
by means of the binary relation
$arc(X,Y)$ (representing the fact that there exits an
arc in $G$ from the generic node $X$ to the generic node $Y$),
and two nodes $s$ and $t$ of $G$,
decide whether $t$ is reachable from $s$ in $G$ or not, that is to say,
decide if there exists a path joining $s$ with $t$ in $G$.

The following $datalog^+(1)$ program
\[ P : \left\{
\begin{array}{l}
r_1 : path(X,Y) \leftarrow arc(X,Y)\\
\\
r_2 : path(X,Y) \leftarrow path(X,Z), arc(Z,Y)
\end{array}
\right. \]
can be employed to solve the reachability problem,
since it holds that $P\cup db_G\entails path(s,t)$ if and only if
$t$ is reachable from $s$ in $G$.
According to Theorem \ref{th:nu_expr},
the rule $r_1$ is translated into the following defaults:
\[ \delta_1^{r_1} \equiv \frac{\overline{arc}(X,Y) : \neg path_1(X,Y)}{\neg path_1(X,Y)},
\delta_0^{r_1} \equiv \frac{: path_1(X,Y)}{path_1(X,Y)}, \mbox{ and} \]
\[ \delta^{r_1} \equiv \frac{path_1(X,Y) : path(X,Y)}{path(X,Y)}, \]
while the rule $r_2$ is translated into the following defaults:
\[ \delta_1^{r_2} \equiv \frac{\overline{arc}(Z,Y) : \neg path_2(X,Z,Y)}{\neg path_2(X,Z,Y)},
\delta_0^{r_2} \equiv \frac{path(X,Z) : path_2(X,Z,Y)}{path_2(X,Z,Y)}, \mbox{ and} \]
\[ \delta^{r_2} \equiv \frac{path_2(X,Z,Y) : path(X,Y)}{path(X,Y)}, \]
and the corresponding NU query form $qf$
is $qf = \langle (D_{qf},W_{db_G}), path(s,t) \rangle$,
with $D_{qf}$ the defaults as above.
The propositional version of $qf$ can be
obtained, once the finite domain of constants $\cal U$
is given, by substituting variables with the constants of $\cal U$
in all the possible consistent ways.

\paragraph*{Observation} For the sake of completeness, we note that, actually,
NU default theories are capable of simulating a limited form of conjunction and negation that could
possibly appear in datalog rules. More precisely, it is possible to correctly
deal with the language $datalog^+(1)$, a generalization of $datalog$ where the
negation of extensional predicates are allowed to occur in rule bodies. Notice,
however, that the translation does not deal with the problem of generating
complement relations on their own, so that to represent input finite structures
via completed databases is needed.

\subsection{The Incremental Lemma}
\label{sect:incr}

This section is devoted to proving the \textit{Incremental Lemma},
which is the basic result used to single out tractable cases of
outlier detection problems.
Alongside, the Incremental
Lemma provides an interesting monotonicity characterization in NMU
theories which is valuable on its own.
\begin{definition}[Proof]\rm
\label{proof} Let $\Delta=(D,W)$ be an NMU default theory, let $l$
be a literal and $E$ a set of literals. A {\em proof} of $l$
w.r.t. $\Delta$ and $E$ is either $l$ by itself, if $l \in W$, or
a sequence of defaults $\delta_1,...,\delta_n$, such that the
following holds:
($1$) $l$ is the consequence of $\delta_n$,
($2$) $\neg l \not\in E$, and
($3$) for each $\delta_i$, $1 \leq i \leq n$, either $\delta_i$ is
prerequisite-free, or $\pre(\delta_i) \in W$, or $\delta_1,...,\delta_{i-1}$
is a proof of $\pre(\delta_i)$ w.r.t. $\Delta$ and $E$.
A proof is {\em minimal} if it is not possible to make it shorter by
deleting a default from it.
\end{definition}
\begin{lemma}\rm \citep{BeDe94} \it
\label{pr} Let $\Delta=(D,W)$ be an NMU default theory, let $l$ be a
literal and $E$ an extension of $\Delta$. Then $l$ is in $E$
iff there is a proof of $l$ w.r.t. $\Delta$ and $E$.
\end{lemma}
\begin{definition}[Satisfaction of an NMU default]\rm
A set of literals $E$ {\em satisfies} an NMU default
$\delta=\frac{y:x}{x}$ iff at least one of the following three
conditions hold: {\rm ($1$)} $y \not\in E$, or {\rm ($2$)} $\neg x
\in E$, or {\rm ($3$)} $x \in E$.
\end{definition}
\begin{theorem}\rm \citep{BeDe94}\it
\label{bd} A
set of literals $E$ is an extension of an 
NMU consistent default theory $(D,W)$ iff the following holds: {\rm
($1$)} $W \subseteq E$, {\rm ($2$)} $E$ satisfies every default
in $D$,
and {\rm ($3$)} every literal in $E$ has a proof w.r.t $(D,W)$ and
$E$.
\end{theorem}
The Incremental Lemma, reported below, characterizes a
monotonic behavior of NMU theories.
\begin{lemma}[The Incremental Lemma]
\label{incremental} Let $(D,W)$ be an NMU default theory, $q$
a literal and $S$ a set of literals such that $W \cup S$ is
consistent and $S$ does not influence $q$ in $(D,W)$. Then the
following hold:
\begin{description}
\item[\it Monotonicity of brave reasoning]
If $q$ is in some extension of $(D,W)$ then $q$ is in some extension
of $(D,W \cup S)$.
\item[\it Monotonicity of skeptical reasoning]
If $q$ is in every extension of $(D,W)$ then $q$ is in every
extension of $(D,W \cup S)$.
\end{description}
\end{lemma}
\begin{proof}
(\textbf{Monotonicity of brave reasoning})
Suppose that for some extension $\tag{E}$ of $(D,W)$, $q \in \tag{E}$, and let
$\sigma= \delta_{1},..., \delta_{k}$ be a proof of $q$ w.r.t $(D,W)$ and $\tag{E}$.
Let $E_\sigma$ be an extension of $\Delta=(\sigma, W \cup S)$. We will show that $\sigma$ is also a proof of $q$ w.r.t. $\Delta$ and $E_\sigma$.
Clearly each prefix of $\sigma$ is a proof w.r.t $(D,W)$ and $\tag{E}$. We prove that each prefix is also a proof w.r.t. $\Delta$ and $E_\sigma$. The proof is by induction on the size of the prefix.
\begin{description}
\item[Case the prefix is of size 1]
In this case, the proof is either a literal that belongs to $W$,
or a default of the form $l:t/t$, where $l \in W$ and $t$ is a literal.
If the proof is a literal that belongs to $W$, it obviously
belongs also to $W \cup S$. Suppose the consequence
of $\delta$ is some literal $t$. It cannot be the case that
a default having a consequence $ \neg t$ belongs to $\sigma$
because $\sigma$ is a proof and $\delta$ is in $\sigma$. It also
cannot be the case that $ \neg t \in W$ because $\sigma $
is a proof w.r.t. $(D,W)$. Last, it also cannot be the case that
$\neg t \in S$ because $S$ does not influence $q$ in $(D,W)$ and
$\delta$ is part of the proof of $q$ w.r.t. $(D,W)$ and $\tag{E}$.
Hence $\delta$ is applicable in $E_\sigma$, and so $t \in E_\sigma$.
\item[Case the prefix is of size greater than 1]
Suppose $\delta$ is the last default in the prefix.
If $\delta$ has a prerequisite $l$, then by the induction
hypothesis, $l \in E_\sigma$. Suppose the consequence of
$\delta$ is some literal $t$.
By arguments similar to the induction base case, $\delta$ is applicable in $E_\sigma$,
and so $t \in E_\sigma$.
\end{description}
Since $\sigma \subseteq D$, by semi-monotonicity of normal default theories, there is an extension
$ E $ of $(D,W \cup S)$ such that $E_\sigma \subseteq E$. Since $q \in E_\sigma$, $q \in E$.

(\textbf{Monotonicity of skeptical reasoning})
Suppose that $q$ is in every extension of $(D,W)$ and
assume conversely that there is an extension $\tag{E}$ of $(D,W \cup S)$ such that $q \notin \tag{E}$. Let $\sigma=\delta_1,...,\delta_n$ be a sequence of generating defaults of $\tag{E}$ as defined in Lemma \ref{ks1}.
$\sigma$ will be
modified so that it will not have defaults with consequences that are influenced by $S$. This is done as follows:
\begin{enumerate}
\item Delete from $\sigma$ all rules of the form $:l/l$ or $t:l/l$ where $l$ belongs to $S$.
\item For $h=1$ to $n$, if $\delta_h$ was not deleted in the previous step and the prerequisite of $\delta_h$ is
not in $W$ and not a consequence of any default which is
before $\delta_h$ in $\sigma$ and was not deleted, then delete $\delta_h$ from $\sigma$.
\end{enumerate}
Let $\sigma_S$ be the sequence of defaults left in $\sigma$ after the modification described above.
Let $E_\sigma$ be an extension of $(\sigma_S, W)$.
\begin{claim}
\label{m}
For every default $\delta$ in $\sigma_S$, the consequence of $\delta$ is in $E_\sigma$.
\end{claim}
Proof:
By induction on the index of $\delta$ in the sequence $\sigma_S$.
\begin{description}
\item[Case the index of $\delta$ is 1]
If $\delta$ has a prerequisite $l$ then, by construction of $\sigma_S$, $l$ must belong to $W$. Suppose the consequence of $\delta$ is some literal $t$. It cannot be the case that a default having a consequence $ \neg t$ belongs to $\sigma_S$ because $\sigma_S$ is a subset of a set of generating defaults and $\delta$ is in $\sigma_S$. It also cannot be the case that $ \neg t \in W$ because $\sigma_S$ is a subset of the generating defaults of $(D,W \cup S)$ and $\delta$ is in $\sigma_S$. Hence, $\delta$ is applicable in $E_\sigma$, and so $t \in E_\sigma$.
\item[Case the index of $\delta$ is greater than 1]
By the induction hypothesis, if $\delta$ has a prerequisite, it has to belong to $E_\sigma$. By arguments similar to the ones given in the induction base part, the consequence of $\delta$ belongs to $E_\sigma$.
\end{description}
By
semi-monotonicity of normal default theories, there is an extension
$E$ of $(D,W)$ such that $E_\sigma \subseteq E$.
By the assumption that $q$ belongs to every extension of $(D,W)$, $q \in E$. Then, by Lemma \ref{pr}, there is a sequence of defaults $\pi=\delta_1,...,\delta_k$ such that $\pi$ is a proof of $q$ w.r.t. $(D,W)$ and $E$. Since $q \notin \tag{E}$, for some $1 \leq j \leq k$, $\delta_j \notin \sigma_S$. Let $i$ be the minimum index such that $\delta_i \in \pi$ and $\delta_i \notin \sigma_S$. It must be the case that $\delta_i = l:t/t$ where $l$ might be empty.
We now consider two cases.
\begin{description}
\item[Case $\delta_i$ is applicable in $\tag{E}$]
If this is the case, $\delta_i$ is one of the generating defaults of $\tag{E}$ or, in other words, $\delta_i \in \sigma$. Since $\delta_i \notin \sigma_S$, and $\delta_i \in \sigma$, it must be the case that $S$ influences $t$ in $(D,W)$. But $\delta_i$ is a default in a proof of $q$ w.r.t. $(D,W)$ and $E$, so $S$ influences $q$ in $(D,W)$, a contradiction.
\item[Case $\delta_i$ is {\em not} applicable in $\tag{E}$]
Since $i$ is the minimal index such that $\delta_i \notin \sigma_S$,
it must be the case that $\neg t$ belongs to $\tag{E}$ (for, otherwise, $\delta_i$
would have been in the generating defaults of $\tag{E}$, and
since $S$ does not influence $t$ in $D$, $\delta_i$ would belong to $\sigma_S$).
Since $\delta_i=l:t/t$ is part of a proof of $q$ w.r.t.
$(D,W)$ and $E$, it must be the case that $ \neg t \notin W$.
Clearly, since $S$ does not influence $t$ in $(D,W)$,
$\neg t$ is not in $S$ either. But $\delta_i$ is {\em{not}}
applicable in $\tag{E}$, which is an extension of $(D, W \cup S)$,
so there must be a rule $\delta \in D$ such that $\delta$ is
applicable in $\tag{E}$ and the consequence of $\delta$ is $\neg t$. So $\delta$ is one of the generating defaults of $\tag{E}$.
Since $S$ does not influence $t$ in $D$, $\delta$ must belong to
$\sigma_S$. By Claim \ref{m} above, $ \neg t$ belongs to
$E_\sigma$, and by semi-monotonicity, $ \neg t$ belongs to $E$.
So $\delta_i=l:t/t$ cannot be part of a proof of $q$ w.r.t. $(D,W)$ and $E$, a contradiction.
\end{description}
\end{proof}

\subsection{On CNF evaluation and the definition of outlier}
\label{sect:technical}

The lemma introduced next shows how
the evaluation of the truth value of a CNF formula
under a specific assignment can be accomplished in a way
that relates it to the definition of outlier.

Let $T$ be a truth assignment to the set $x_1, \ldots, x_n$ of
boolean variables. Then $Lit(T)$ denotes the set of literals $\{
\ell_1, \ldots, \ell_n \}$, such that $\ell_i$ is $x_i$ if $T(x_i) =
{\bf true}$ and it is $\neg x_i$ if $T(x_i) = {\bf false}$, for
$i=1,\ldots,n$.

Let $L$ be a consistent set of literals. Then ${\cal T}_L$ denotes
the truth assignment to the set of letters (boolean variables)
occurring in $L$ such that, for each positive literal $p \in L$,
${\cal T}_L(p) = {\bf true}$, and for each negative literal $\neg p
\in L$, ${\cal T}_L(p) = {\bf false}$.

Lemma \ref{lemma:intract} below
states the connection between  the evaluation of the truth value of a CNF formula under
a specific assignment and the definition of outliers.

\begin{lemma}\label{lemma:intract}
For each boolean formula $\Phi$ in 3CNF having $m$ conjuncts, there exists
a NU propositional default theory $(D(\Phi),W(\Phi))$
whose size is polinomially bounded in $\Phi$,
a set of literals $S(\Phi) \subseteq W(\Phi)$,
and a set of letters $c_1, \ldots, c_m$ occurring in $D(\Phi)$,
such that $\Phi$ is satisfiable if and only if
$(D(\Phi),W(\Phi)_{S(\Phi)}) \models \neg (S(\Phi) \cup \{c_1, \ldots, c_m\})$.
\end{lemma}
\begin{proof}
Let $\Phi = f(X)$ be a boolean formula in 3CNF,
where $X = x_1, \ldots, x_n$ is a set of variables, and $f(X) = C_1
\wedge \ldots \wedge C_m$, with $C_j = t_{j,1} \vee t_{j,2} \vee
t_{j,3}$, and each $t_{j,1},t_{j,2},t_{j,3}$ is a literal on the
set $X$, for $j=1,\ldots,m$.
Let $\Delta(\Phi) = (D(\Phi),W(\Phi))$ be a NU
default theory associated with
$\Phi$, where $W(\Phi)$ is the set $\{ x_1,\ldots,x_n\}$ of letters
and
$D(\Phi)$ is the set of defaults $D_1 \cup D_2 \cup D_3$, with:
\begin{eqnarray*}
D_1 & = & \left\{
\delta^{(1)}_{1,i} = \frac{x_i:\neg y_i}{\neg y_i},
\delta^{(1)}_{2,i} = \frac{:y_i}{y_i}
\mid i=1,\ldots,n \right\},\\
D_2 & = & \left\{
\delta^{(2)}_{j,k} = \frac{ \sigma(t_{j,k}) : \neg c_j}{\neg c_j}
\mid j=1,\ldots,m;~k=1,2,3 \right\},\\
D_3 & = & \left\{
\delta^{(3)}_{1,i} = \frac{: \neg x_i}{\neg x_i} 
\mid i=1, \ldots, n \right\}, 
\end{eqnarray*}
where $c_1$, $\ldots$, $c_m$ and $y_1$, $\ldots$, $y_n$ are new
letters distinct from those occurring in $\Phi$, and $\sigma(x_i) =
x_i$ and $\sigma(\neg x_i) = y_i$, for $i=1,\ldots,n$.
Let $R$ be a subset of $\{x_1,\ldots,x_n\}$. In the rest of the
proof, $\sigma(R)$ will denote the set $\{ \sigma(x) \mid x \in R \}$.
Moreover, $\sigma^{-1}$ will denote the inverse of $\sigma$, that is
to say $\sigma^{-1}(x_i) = x_i$ and $\sigma(y_i) = \neg x_i$, and
$\sigma^{-1}(R) = \{ \sigma^{-1}(x) \mid x \in R \}$.
%
%

Next, it is shown that there exists a set of literals $S(\Phi)\subset W(\Phi)$
such that $(D(\Phi),W(\Phi)_{S(\Phi)}) \models \neg (S(\Phi)\cup\{c_1,\ldots,c_m\})$
if and only if $\Phi$ is satisfiable.

First, the following claims are proved.
\begin{claim}\label{claim3}
Let $S\subseteq W(\Phi)$.
For each extension $\cal E$ of $(D(\Phi),W(\Phi)_S)$ it is the case that
\begin{itemize}
\item
for each $x_i\in S$, $\neg x_i\in{\cal E}$ and $y_i\in{\cal E}$, and
\item
for each $x_i\not\in S$, $x_i\in{\cal E}$.
\end{itemize}
\end{claim}
\begin{claimproof}{claim3}
Any extension $\cal E$ of $(D(\Phi),W(\Phi)_S)$ is such that
($i$)
for each $x_i\in S$, $x_i \not\in {\cal E}$ and $\neg x_i \in {\cal E}$,
since $x_i\not\in W(\Phi)_S$,
no letter $x_i$ appears in the consequence of a default in $D(\Phi)$,
and due to defaults $\delta^{(3)}_{1,i}$;
($ii$)
for each $x_i\in S$, $y_i \in {\cal E}$, due to defaults $\delta^{(1)}_{2,i}$; and
($iii$)
for each $x_i\not\in S$, $x_i \in {\cal E}$, since $x_i\in W(\Phi)_S$.
\end{claimproof}

\begin{claim}\label{claim2}
Let $S\subseteq W(\Phi)$.
Then, there exists an extension ${\cal E}^\ast$ of $(D(\Phi),W(\Phi)_S)$
such that
\begin{itemize}
\item
for each $x_i\in S$, $\neg x_i\in{\cal E}^\ast$ and $y_i\in{\cal E}^\ast$, and
\item
for each $x_i\not\in S$,
$x_i\in{\cal E}^\ast$ and $\neg y_i\in{\cal E}^\ast$.
\end{itemize}
\end{claim}
\begin{claimproof}{claim2}
Note that for each $x_i\not\in S$ either
$\neg y_i\in {\cal E}$ (due to defaults $\delta^{(1)}_{1,i}$)
or
$y_i \in {\cal E}$
(due to defaults $\delta^{(1)}_{2,i}$).
The proof then follows from Claim \ref{claim3}.
\end{claimproof}

We next resume to the main Lemma proof.

(If part) Assume that a subset $S(\Phi)\subseteq W(\Phi)$ exists
such that $(D(\Phi)$, $W(\Phi)_{S(\Phi)}) \models \neg(S(\Phi)\cup\{c_1,\ldots,c_m\})$.
Note that the set $S(\Phi)$ may contain some letters from the set
$\{x_1,\ldots,x_n\}$.
Further,
the defaults in the set $D_1$ serve the purpose
of introducing the letters $y_i$ 
in the current extension of the theory $(D(\Phi),W(\Phi)_S)$.
In particular,
the letter $y_i$ is intended to represent the negation
of letter $x_i$ and it is introduced since in NU theories negated
literals cannot be specified as the prerequisite of a default.

Moreover, the
defaults in the set $D_2$ evaluate the CNF formula $\Phi$
by using the truth value assignment encoded by the letters $x_i$ and $y_i$
belonging to the current extension of the theory $(D(\Phi),W(\Phi)_S)$.
In particular, if the $j$th conjunct $C_j$ of $\Phi$ is true,
then the literal $\neg c_j$ belongs to the current extension.

By Claim \ref{claim2}, there exists an exension ${\cal E}^\ast$
of $(D(\Phi),W(\Phi)_S)$ such that the set of literals
$R = \sigma^{-1}({\cal E}^\ast \cap (X \cup Y))$ is consistent and precisely
encodes a truth value assignment
for the variables in the set $X$, namely $T_R$.

To conclude, since $(D(\Phi),W(\Phi)_S) \models \neg c_1 \wedge \ldots \wedge \neg c_m$,
by the defaults in the set $D_2$, $T_R$ is a truth value assignment
to the variables in the set $X$ that makes the formula $f(X)$ true and,
hence, $\Phi$ is satisfiable.

(Only If part)
Assume now that the formula $\Phi$ is satisfiable, and
let $T_X$ be a truth value assignment to the variables in the set
$X$ that makes $f(X)$ true.
Next, we prove that for
$S(\Phi)=\{x_i \mid T_X(x_i) = \textbf{false} \}$
it holds that $(D(\Phi),W(\Phi)_{S(\Phi)}) \models \neg(S(\Phi)\cup\{c_1,\ldots,c_m\})$.

Let $\cal E$ be a generic extension of $(D(\Phi),W(\Phi)_{S(\Phi)})$.
As for the letters $x_i$ belonging to $S(\Phi)$,
due to defaults $\delta^{(3)}_{1,i}$,
$\neg x_i$ belongs to ${\cal E}$.

It remains to show that $\neg c_j\in{\cal E}$, for each $j\in\{1,\ldots,m\}$.
By Claim \ref{claim3}, for each extension $\cal E$ of
$(D(\Phi),W(\Phi)_{S(\Phi)})$, the set
$Q = {\cal E}^\ast \cap (X \cup Y)$
is a superset of $\sigma(\Lit(T_X))$.
That it to say, it might be the case that
both $x_i$ and $y_i$ are in the set $Q$, for a letter $x_i$ not in $S(\Phi)$.

Since the defaults in the set $D_2$ encode the boolean formula
obtained from $\Phi$ by substituting each negative literal $\neg x_i$
with the positive literal $y_i$,
having both $x_i$ and $y_i$ in $Q$ corresponds to assuming
that both the literals $x_i$ and $\neg x_i$ are, so as to say, true within $\Phi$.

Indeed, as already observed, $\Phi$ is evaluated by
means of a different formula, say $\Phi'$, obtained from $\Phi$
by substituting each occurrence of the negated literal $\neg x_i$ with the
letter $y_i$ (see the defaults in $D_2$).
If the literals in $Q$ form a superset of the literals associated
with a truth value assignment making $\Phi$ true, they also encode
a satisfying truth assignment that makes the formula $\Phi'$ true.
Thus, by virtue of the defaults in $D_2$, each extension $\cal E$ of
$(D(\Phi),W(\Phi)_{S(\Phi)})$ is such that $\neg c_j$ belongs to $\cal E$,
for $j=1,\ldots,m$.  
\end{proof}

\begin{figure}
\begin{center}
\includegraphics[width=300px,height=150px]{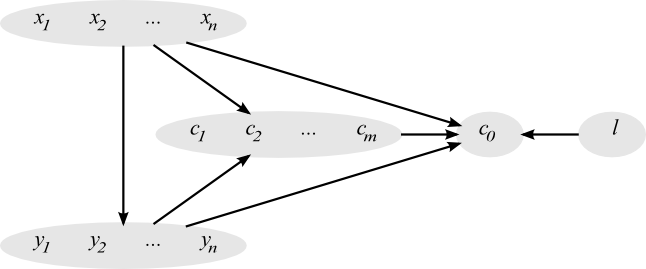}
\end{center}
\caption{The atomic dependency graph of the theory $\Delta'(\Phi)$ in Theorem \ref{th:out_acy_rec}.}
\label{fig:out_acy_rec_graph}
\end{figure}

\section{The tractabilty/intractability frontier for Outlier-Witness Recognition}
\label{sect:out_wit_prob}

This section is devoted to the problem Outlier-Witness Recognition, that is,
given default theory $\Delta = (W,D)$ and two disjoint sets $L$ and $S$
of literals from $W$, is it true that
$L$ is an outlier in $\Delta$ with witness $S$?.
We begin by recalling the following result from \citep{ABP10}:
\begin{theorem}[from \citep{ABP10}]
Over NU default theories, the problem Outlier-Witness Recognition
is solvable in polynomial time.
\end{theorem}

Therefore, in order to draw the sought frontier,
it remains to show the intractability of Outlier-Witness Recognition
in a language setting which is (possibly only slightly) more general than
that of NU theories, which is provided by the following result
which considers the
least general
NMU theories.

\begin{theorem}
Over quasi-acyclic NMU default theories, the problem Outlier-Witness Recognition is {\rm co-NP-hard}.
\end{theorem}
\begin{proof}
The proof is obtained via a reduction of the entailment problem
for quasi-acyclic NMU theories which was proved to be $\CONP$-complete
in Section \ref{th:nmu_entail} above. Consider an instance of the
entailment problem relative to a quasi-acyclic NMU theory $\Delta = (W,D)$
and one single literal $l$. Consider, now, the theory
$$\Delta' = (D',W') = \left(W \cup \{\neg f, g\} ,D \cup \left\{\frac{: \neg g}{\neg g};
\frac{\neg g: \neg f}{\neg f}; \frac{l:l'}{l'}; \frac{l':f}{f} \right\} \right),$$
where $f$, $g$ and $l'$ are letters not occurring in $\Delta$.
Clearly, $\Delta'$ is quasi-acyclic NMU and can be constructed in time
polynomial in the size of $\Delta$.

Now we have: $\Delta \models l$
implies $\Delta' \models l$ that, in turn, implies $\Delta' \models l'$.
From this, it follows that that $(D',W'_{\{\neg f\}}) \models f$ and,
since $(D',W'_{\{\neg f\},\{g\}}) \models \neg g$, we also have
$(D',W'_{\{\neg f\},\{g\}} \not\models f$, which finally means that $\{g\}$
is an outlier set in $\Delta'$ with witness $\{\neg f\}$.

As for the opposite direction,
assume that $\{g\}$ is an outlier set in $\Delta'$ with witness $\{\neg f\}$.
This means that $(D',W'_{\{\neg f\}}) \models f$ holds.
Therefore, since neither $f$ nor $g$ occur in $\Delta$, by minimality,
this implies in turn that $\Delta \models l'$ and, finally, that $\Delta \models l$.
This completes the proof.
\end{proof}

\section{The tractabilty/intractability frontier for Outlier Recognition}
\label{sect:out_rec_prob}

This section deals with the problem Outlier Recognition, that is,
given default theory $\Delta = (W,D)$ and a set $L$ of literals from $W$,
is there any set $S$ of literals from $W$, disjoint from $L$, such that
$L$ is an outlier in $\Delta$ with witness $S$?.

The section is organized as follows.
Section \ref{sect:gen_out_rec} considers the Outlier Recognition
problem over quasi-acyclic NU theories.
Section \ref{sect:strong_out} introduces the notion of strong outlier.
Subsequent Section \ref{sect:strong_out_wit_rec} studies
the Outlier-Witness recognition problem on the notion of
strong outlier. Finally, Section \ref{sect:strong_out_rec}
takes into account the complexity of the Outlier Recognition
problem on strong outliers both on cyclic and quasi-acyclic NU theories.

\subsection{General outliers and the Outlier Recognition problem}
\label{sect:gen_out_rec}

We begin, as in Section \ref{sect:out_wit_prob}, by recalling a theorem from \citep{ABP10},
but this time one stating an intractability result:

\begin{theorem}[from \citep{ABP10}]
Over NU default theories, the problem Outlier Recognition is $\NP$-complete.
\end{theorem}

By exploiting the result proved in Section \ref{sect:technical},
it is possible to sharpen the result presented above by showing that
Outlier Recognition remains intractable even over the (quite simple) class of {\em quasi-acyclic} NU theories, which is proven next.

\begin{theorem}\label{th:out_acy_rec}
\textit{Outlier Recognition} for quasi-acyclic NU default theories is $\rm NP$-complete.
\end{theorem}
\begin{proof}
(Membership) Membership in NP immediately follows from Theorem 3.6 of \citep{ABP10}.

(Hardness)
Let $\Phi$ be a boolean formula in 3CNF
(as described in Lemma \ref{lemma:intract}).
Checking 3CNF formulae satisfiability is a well known NP-complete problem.
\citep{pap94}.

Consider the default theory $(D(\Phi),W(\Phi))$ described
in the proof of Lemma \ref{lemma:intract}.
In order to prove the NP-hardness of the problem at hand, the NU
default theory $\Delta'(\Phi) = (D'(\Phi),W'(\Phi))$
is associated with
$\Phi$, where $W'(\Phi)$ is the set
$$W(\Phi) \cup \{ \neg l,c_0, c_1,\ldots,c_m \}$$
of letters, with
$l,c_0,c_1,\ldots,c_{m}$ being new letters distinct from those occurring in
$\Phi$, and $D'(\Phi)$ is the set of defaults
$D(\Phi) \cup D_4\cup D_5\cup D_6$, with:
\begin{eqnarray*}
D_4 & = & \left\{
\delta^{(4)}_{1,i} = \frac{x_i : \neg c_{0}}{\neg c_{0}},
\delta^{(4)}_{2,i} = \frac{y_i : \neg c_{0}}{\neg c_{0}}
\mid i=1, \ldots, n \right\},\\
D_5 & = & \left\{
\delta^{(5)}_j = \frac{c_j : c_{0}}{c_{0}}
\mid j=1,\ldots,m \right\}, \mbox{ and}\\
D_6 & = & \left\{
\delta^{(6)}_1 = \frac{:l}{l},
\delta^{(6)}_2 = \frac{l:c_{0}}{c_{0}}
\right\}.
\end{eqnarray*}
The theory $\Delta'(\Phi)$ can be built in polynomial time
and, moreover, $W'(\Phi)$ is consistent.

Now we can observe that the theory $\Delta'(\Phi)$
has tightness equal to $1$ --- the atomic dependency graph of $\Delta'(\Phi)$ is shown
in Figure \ref{fig:out_acy_rec_graph}.
Next, it is shown that $L=\{\neg l\}$ is an outlier in $\Delta'(\Phi)$
if and only if $\Phi$ is satisfiable.

(Only If part)
Assume that $L=\{\neg l\}$ is an outlier in $\Delta'(\Phi)$.
Since $L$ is an outlier, there exists
an outlier witness set
$S\subseteq W'(\Phi)\setminus L$ for $L$.

Now we show that the set $S$ contains all the literals
in the set $\{c_0$, $c_1$, $\ldots$, $c_{m}\}$.
The only literal of $W'(\Phi)\setminus L$
whose negation
is not entailed by $(D'(\Phi),W'(\Phi)_{S,L})$
is $c_{0}$, since this theory
entails $c_{0}$ by the defaults in $D_6$.
Thus, $c_{0}$ belongs to $S$,
for otherwise $\{\neg l\}$ would not be an outlier set.
Since $(D'(\Phi),W'(\Phi)_S)\models\neg S$ and $c_{0}\in S$,
by the defaults in $D_5$, $S$
contains the set $\{c_0,c_1,\ldots,c_{m}\}$,
for otherwise $(D'(\Phi),W'(\Phi)_S)\models\neg c_{0}$ cannot be true.

Let $C$ be $\{c_1,\ldots,c_m\}$ and
let $R$ be $S\setminus(C\cup\{c_0\})$.
Next we show that there does not exist an extension $\cal E$ of
$(D(\Phi),W(\Phi)_R)$ such that ${\cal E}\not\supseteq \neg (R\cup C)$
and, hence, that $(D(\Phi),W(\Phi)_R) \models \neg (R\cup C)$.

First, note that $W'(\Phi)_{S}$ equals $W(\Phi)_R\cup\{\neg l\}$,
while $D(\Phi)\subseteq D'(\Phi)$.
Since, $\neg l$ does not occur in any default
of the theory $(D(\Phi),W(\Phi))$ and since $\neg l$ occurs in
$W'(\Phi)$, then the semi-monotonic relationship still
holds between theories $(D(\Phi),W(\Phi)_R)$ and
$(D'(\Phi),W'(\Phi)_{R'})$.

Assume that there exists
an extension $\cal E$ of
$(D(\Phi),W(\Phi)_R)$ such that ${\cal E}\not\supseteq \neg (R\cup C)$.
Then, there exists either a literal $\neg x_{i'}\not\in{\cal E}$
(${i'}\in\{1,\ldots,n\}$) or
a literal $\neg c_{j'}\not\in{\cal E}$ ($j'\in\{1,\ldots,m\}$).
Moreover, by semi-monotonicity, there exists
an extension ${\cal E}'$ of $(D'(\Phi),W(\Phi)_{S})$
such that ${\cal E'}\supseteq{\cal E}$.
However, since no literal $\neg x_i$ ($1\le i\le n$) and
$\neg c_j$ ($1\le j\le m$) occur
in any default in the set $D_4\cup D_5\cup D_6$,
the missing literal ($x_{i'}$ or $c_{j'}$)
does not belong to ${\cal E}'$ as well and
$(D'(\Phi),W'(\Phi)_{S})\not\models \neg S$, a contradiction.
Thus, it can be concluded that
$(D(\Phi),W(\Phi)_R)\models\neg (R \cup C)$
and, by Lemma \ref{lemma:intract}, that
the formula $\Phi$ is satisfiable.

(If part) Assume now that the formula $\Phi$ is satisfiable.
Then, by Lemma \ref{lemma:intract}, there exists
a set of literals $R \subseteq \{x_1,\ldots,x_n\}$
such that $(D(\Phi),W(\Phi)_R)\models \neg(R\cup C)$.

Let $S$ be $R\cup C\cup \{c_0\}$.
We show next that $(D'(\Phi),W'(\Phi)_S)\models \neg S$.
By semi-monotonicity, it follows that there exists at least
one extension ${\cal E}'$ of $(D'(\Phi),W'(\Phi)_S)$
such that ${\cal E}'\supseteq \neg(R\cup C)$.
Moreover,
since no literal occurring in the set of defaults
$D(\Phi)$ also occurs in the consequence of any default in
$D_4\cup D_5\cup D_6$, each
extension ${\cal E}'$ of $(D'(\Phi),W'(\Phi)_S)$
is such that ${\cal E}'\supseteq \neg(R\cup C)$.
Then, for each extension ${\cal E}'$ of $(D'(\Phi),W'(\Phi)_S)$, there exists at least
one letter $x_i$ or $y_i$ in ${\cal E}'$. Indeed, if $x_i\not\in S$, then $x_i\in{\cal E}'$.
Otherwise, $y_1,\ldots,y_n\in{\cal E'}$. Thus, by virtue of the defaults in the set $D_4$,
the literal $\neg c_0$ belongs to ${\cal E}'$.
It can be concluded that $(D'(\Phi),W'(\Phi)_S)\models \neg S$.

Consider now the theory $(D'(\Phi),W'(\Phi)_{S,\{\neg l\}})$.
By virtue of the defaults in the set $D_6$, there exists at least one extension
${\cal E}'$ of this theory such that $c_0\in {\cal E}'$ and, thus,
$(D'(\Phi),W'(\Phi)_{S,\{\neg l\}}) \not\models \neg S$.
In other words, $L=\{\neg l\}$ is an outlier set with
associated outlier witness set $S$ in $\Delta'(\Phi)$.
%
\end{proof}

Unfortunately, the theorem reported above
confirms that detecting outliers, even in default theories
as structurally simple as quasi-acyclic NU ones, remains intractable.
Therefore, unless one is interested in looking into even simpler fragments of default logics,
in order to attain tractability, it is necessary to resort to a stronger notion of outlier,
allowing to single out a strict, yet interesting, subset of the outlier as captured by the general definition we have used up to this point.
This stronger notion is introduced next.

\subsection{Strong outliers}
\label{sect:strong_out}

In order to introduce strong outliers, we preliminarily
notice that conditions 1 and 2 of the Definition \ref{outlierD} can be rephrased as follows:
\begin{enumerate}
\item
$(\forall \ell\in S) (D,W_S) \models \neg \ell$, and
\item
$(\exists \ell\in S) (D,W_{S,L}) \not\models \neg \ell$.
\end{enumerate}
In other words, condition $1$ states that the negation of every literal $\ell\in S$
must be entailed by $(D,W_S)$ while, according to condition $2$, it
is sufficient that
just one literal $\ell\in S$ exists whose negation is not entailed by $(D,W_{S,L})$.
In order to strengthen this definition, it is therefore natural (and, as we shall show, enough) to modify Condition $2$, thereby
obtaining the following definition of \textit{strong outlier} set.
\begin{definition}[Strong Outlier]
\label{stroutD} Let $\Delta=(D,W)$ be a propositional default
theory and let $L \subset W$ be a set of literals.
If there exists a non-empty subset $S$ of $W_L$ such that:
\begin{enumerate}
\item
$(\forall \ell\in S) (D,W_S) \models \neg \ell$, and
\item
$(\forall \ell\in S) (D,W_{S,L}) \not\models \neg \ell$
\end{enumerate}
then $L$ is a {\em strong outlier} set in $\Delta$ and $S$ is a
{\em strong outlier witness} set for $L $ in $\Delta$.
\end{definition}
The following proposition shows that strong outliers indeed form a subset of outliers and is immediately proved:
\begin{proposition}\label{prop:strong}
If $L$ is a strong outlier set and $W$ its strong witness, then $L$ is an outlier set and $W$ its witness.
\end{proposition}
Note that, in general the converse of Proposition \ref{prop:strong}
does not hold.

An abstract example is presented below that provides some intuition about cases in which the newly introduced notion of strong outlier admittedly fits better than general outliers
(the reader is referred to \citep{ABP08} for compelling examples of application for general outliers).
Consider a default theory
$\Delta=(D,W)$ where $D$ is the set of defaults
$$\left\{ \frac{a_1 : b_1}{b_1},\frac{a_2 : b_2}{b_2},\ldots,\frac{a_n : b_n}{b_n}\right\},$$
and $W$ is the set of observations $\{a_1, a_2,....,a_n, \neg b_1, ..., \neg b_n\}$.

According to the standard definition of outlier,
we get that for any possible nonempty subset $L$ of $\{a_1, a_2,...,a_n\}$ and any nonempty subset $S$ of $\{\neg b_1, \neg b_2,...,\neg b_n\}$ that includes at least one literal $\neg b_i$ with $a_i \in L$, $S$ is a witness for $L$ being an outlier. So, even if a certain default $\frac{a_j :b_j}{b_j}$
is not related to another default $\frac{a_k :b_k}{b_k}$, with $j \neq k$
(for example, one is about birds, the other is about students),
anyway
$\{\neg b_j, \neg b_k\}$ is a witness set for the outlier set $\{a_j\}$.

Vice versa, strong outliers establish a closer
correspondence between the outliers and their witness sets.
In the example above, the set $\{\neg b_j, \neg b_k\}$ is {not} a strong witness for $\{a_j\}$, but $\{\neg b_j\}$ is.

\subsection{Strong outliers and the Outlier-Witness Recognition Problem}
\label{sect:strong_out_wit_rec}

In order to mark the tractability landscape of the \textit{Strong Outlier Detection}
problems,
we preliminarily provide two results, the former one regarding the tractability
of the outlier-witness recognition problem and the latter one pertaining to its intractability
\begin{lemma}\label{th:out_wit_rec}
\textit{Strong Outlier-Witness Recognition}
on propositional {\em NU} default theories is
in $\rm P$.
\end{lemma}
\begin{proof}
The proof is immediate since the statement follows from the definition of strong outlier set
(Definition \ref{stroutD}) and the fact that the entailment problem on
propositional NU default theories is polynomial time
solvable (as proved in \citep{KaSe91,Ben02}).
\end{proof}
\begin{lemma}\label{th:out_wit_rec_df}
Strong Outlier-Witness Recognition on propositional {\em DF} default theories
is NP-hard.
\end{lemma}
\begin{proof}
The statement follows from the reduction employed in Theorem 4.6 of \citep{ABP08},
where it is proved that given two DF default theories $\Delta_1=(D_1,\emptyset)$ and
$\Delta_2=(D_2,\emptyset)$,
and two letters $s_1$ and $s_2$, the problem $q$ of deciding
whether $((\Delta_1 \models s_1) \wedge (\Delta_2\models s_2))$ is valid can be reduced
to the outlier-witness problem; that is, to the
problem of deciding whether
$L=\{s_2\}$ is an outlier having witness set $S=\{\neg s_1\}$
in the theory $\Delta(q)$, where $\Delta(q)=(D(q),W(q))$
is the propositional DF default theory
with $D(q) = \{ \frac{s_2\wedge \alpha: \beta}{\beta} \mid \frac{\alpha:\beta}{\beta}\in D_1\} \cup D_2$ and $W(q)=\{\neg s_1,s_2\}$.
Since the former problem is NP-hard, it follows from the reduction
that the latter problem is NP-hard as well.
In order to complete the proof, we note that
a singleton witness set is always a strong witness set and, hence,
the above reduction immediately applies to strong outliers as well.
\end{proof}

\subsection{Strong outliers and the Outlier Recognition Problem}
\label{sect:strong_out_rec}

Now that we have introduced the notion of strong outlier, we are ready to study
the implication of adopting such stronger notion on the complexity of the
problem Outlier Recognition.
First of all, using most of the technical machinery presented in previous
section, we are able to prove a first tractability result concerning this
outlier detection problem, as shown in the following. We begin by presenting a
further technical lemma, which provides a characterization
of minimal strong outlier witness sets in propositional NMU theories,

\begin{lemma}
\label{sccsize}
Let $(D,W)$ be a consistent NMU default theory and
let $L$ be a set of literals in $W$. Then $L$ is a strong outlier set in $(D,W)$
iff there exists an outlier witness set $S$ for $L$ in $(D,W)$ such that
$\lett{S}$ is a subset of a SCC in the atomic dependency graph of $(D,W)$.
\end{lemma}
\begin{proof} Let $L$ be a set of literals in $W$ and let $S$ be an
outlier witness set for $L$ in $(D,W)$. By definition, the following
must be true:
\begin{enumerate}
\item $(\forall \ell\in S) (D,W_S) \models \neg \ell$, and
\item $(\forall \ell\in S) (D,W_{S,L}) \not\models \neg \ell$.
\end{enumerate}
We can partition $S$ into disjoint sets $S_1,\ldots,S_n$ such that the
following holds:
\begin{enumerate}
\item
$\union_{i=1}^{n} S_i =S$.
\item
For each $1 \leq i \leq n$, if $l$ and $q$ are in $S$ then $l \in
S_i$ and $q \in S_i$ if and only if $\lett{l}$ and $\lett{q}$ are in
the same SCC in the atomic dependency graph of $(D,W)$.
\item
For each $q,l \in S$ and for each $1 \le i \leq n-1$, if $l \in
S_{i}$ and $q \in S_{j}$ for some $2 \leq j \leq n$ such that $i
<j$, then there is no path in the atomic dependency graph from
\lett{q} to \lett{l} (that is, the $S_i$'s are ordered according to
the the reachability relationship induced on the the atomic dependency graph).
\end{enumerate}
Next we show that the following holds:
\begin{enumerate}
\item $(\forall \ell\in S_1) (D,W_{S_1}) \models \neg \ell$, and
\item $(\forall \ell\in S_1) (D,W_{S_1,L}) \not\models \neg \ell$.
\end{enumerate}
That means that $S_1$ is a strong outlier witness set for $L$. Since $S_1$
is a subset of an SCC in the atomic dependency graph of $(D,W)$ this
will complete the proof.

We first consider the condition
$(\forall\ell\in S_1) (D,W_{S_1}) \models \neg \ell$.
It is given that $(\forall\ell\in S) (D,W_S) \models \neg \ell$. In other words,
$(\forall\ell\in S) (D,W_{S_1,...,S_n}) \models \neg \ell$ so that clearly
$(\forall\ell\in S_1) (D,W_{S_1,...,S_n}) \models \neg \ell$. Since there is no path in
the atomic dependency graph from any letter of a literal in
$S_2,\ldots,S_n$ to a letter of a literal in $S_1$, we can use the
Incremental Lemma to conclude that $(\forall\ell\in S_1) (D,W_{S_1}) \models \neg \ell$.

Now consider the condition $(\forall\ell\in S_1) (D,W_{S_1,L}) \not\models \neg \ell$.
Since $(\forall\ell\in S) (D,W_{S,L}) \not\models \neg \ell$, clearly
$(\forall\ell\in S_1) (D,W_{S,L})\not\models \neg \ell$.
Note that it cannot be the case that $\exists\ell\in S_1$ such that
$(D,W_{S,L}) \not\models \neg \ell$ but
$(D,W_{S,L} \cup \{S_2,...,S_n\}) \models \neg \ell$, since there is
no path in the atomic dependency graph from $S_2,\ldots,S_n$ to $S_1$.
Hence $(\forall\ell\in S_1), (D,W_{S,L} \cup \{S_2,...,S_n\}) \not\models \neg \ell$ or,
equivalently, $(\forall\ell\in S_1) (D,W_{S_1,L} ) \not\models \neg \ell$.
\end{proof}

Armed with the above result, we are now able to prove the following:

\begin{theorem}\label{th:strout_acy_rec}
\textit{Strong Outlier Recognition} for NU quasi-acyclic default theories is in $\rm P$.
\end{theorem}
\begin{proof}
Given a NU default theory $(D,W)$ of tightness $c$ and a
set of literals $L$ from $W$, by Lemma \ref{sccsize}
a minimal outlier witness set $S$ for
$L$ in $(D,W)$ has a size of at most $c$, since $c$ is the maximum
size of an SCC in the atomic dependency graph of $(D,W)$.
Thus, the strong outlier recognition problem
can be decided by solving the strong outlier-witness recognition
problem for each subset $S$ of literals in $W_L$ having a size of at most $c$.
Since the latter problem is polynomial time solvable (by Theorem \ref{th:out_wit_rec})
and since the number of times it has to be evaluated,
that is $O(|W|^c)$, is polynomially related to the size of the input,
then the procedure solves the strong
outlier recognition problem in polynomial time. 
\end{proof}

Therefore, on the one side now we know that strong outlier recognition can be
attained in polynomial time over quasi-acyclic NU theories and, on the other side, we
know that (general) outlier recognition is intractable over the same kind of
theories. Therefore, the natural question arises of whether theorem
\ref{th:strout_acy_rec} can be strengthened as to refer to more general forms of
theories. The following result demonstrate that this is unfortunately not the
case.

\begin{theorem}\label{th:strout_cyc_rec}
\textit{Strong Outlier Recognition} for NU (cyclic) default theories is $\rm NP$-complete.
\end{theorem}
\begin{proof}
(Membership) Since strong outliers are a subset of general outliers,
the result immediately follows from
Theorem 3.6 of \citep{ABP10}.

(Hardness)
The proof refers, again, to the satisfiability problem for 3CNF. Let $\Phi$ be one such a formula and consider
the default theory $(D(\Phi),W(\Phi))$ described in the proof of Lemma \ref{lemma:intract}.
In order to prove the NP-hardness of the problem at hand, the NU
default theory $\Delta'(\Phi) = (D'(\Phi),W'(\Phi))$
is associated with
$\Phi$, where $W'(\Phi)$ is the set
$$W(\Phi) \cup \{ \neg l, c_1,\ldots,c_m \}$$
of letters, with
$l,c_1,\ldots,c_{m}$ being new letters distinct from those occurring in
$\Phi$, and $D'(\Phi)$ is the set of defaults
$D(\Phi) \cup D_4\cup D_5\cup D_6$, with:
\begin{eqnarray*}
D_4 & = & \left\{
\delta^{(4)}_i = \frac{f : x_i}{x_i},
\mid i=1, \ldots, n \right\},\\
D_5 & = & \left\{
\delta^{(5)}_{1,j} = \frac{c_j :f}{f},
\delta^{(5)}_{2,j} = \frac{f : c_j}{c_j}
\mid j=1,\ldots,m \right\}, \mbox{ and}\\
D_6 & = & \left\{
\delta^{(6)}_1 = \frac{:l}{l},
\delta^{(6)}_2 = \frac{l:f}{f}
\mid j=1,\ldots,m
\right\}.
\end{eqnarray*}
The theory $\Delta'(\Phi)$ can be built in polynomial time
and, moreover, $W'(\Phi)$ is consistent.
Because of the defaults in the set $D_4$, the theory $\Delta'(\Phi)$ is cyclic.

Next it is shown that $L=\{\neg l\}$ is a strong outlier in $\Delta'(\Phi)$
if and only if $\Phi$ is satisfiable.
The following claim will be useful towards this end.

\begin{claim}\label{claimf}
Let $R$ be a subset of $W'(\Phi)$. If there exists an extension $\cal E$ of $(D'(\Phi),W'(\Phi)_R)$ such that
the letter $f$ belongs to $\cal E$ then, for each $\ell\in R\setminus\{\neg l\}$, $(D'(\Phi),W'(\Phi)_R)\not\models\neg\ell$.
\end{claim}
\begin{claimproof}{claimf}
Note that if the letter $f$ can be entailed then,
for each $\ell\in R\setminus\{\neg l\}$,
by virtue of defaults $\delta^{(4)}_i$ and $\delta^{(5)}_{2,j}$,
there exists at least one extension ${\cal E}_\ell$ of
$(\Delta'(\Phi),W'(\Phi)_R)$ such that $\ell\in {\cal E}_\ell$
and, consequently, $(\Delta'(\Phi),W'(\Phi)_R)\not\models \neg\ell$.
\end{claimproof}

We can now resume to the proof of Theorem \ref{th:strout_cyc_rec}.

(Only If part)
Assume that $L=\{\neg l\}$ is a strong outlier in $\Delta'(\Phi)$.
Since $L$ is a strong outlier, there exists
a strong outlier witness set
$S\subseteq W'(\Phi)\setminus L$ for $L$.

By Claim \ref{claimf}, the letter $f$ cannot belong to any extension
of $(D'(\Phi),W'(\Phi)_S)$,
for otherwise $S$ is not a (strong) outlier witness set
inasmuch as Condition 1 of Definition \ref{outlierD} is not satisfied.
Thus, because of defaults $\delta^{(5)}_{1,j}$,
it is the case that the set $S$ contains all the literals
in the set $\{c_1$, $\ldots$, $c_{m}\}$.

Let $C$ be $\{c_1,\ldots,c_m\}$ and let $R$ be $S\setminus C$.
By using exactly the same argument as that employed in the Only-If part of Theorem \ref{th:out_acy_rec}, it can be proved that
$(D(\Phi),W(\Phi)_R) \models \neg (R\cup C)$,
and, by Lemma \ref{lemma:intract},
it can be concluded that
the formula $\Phi$ is satisfiable.

(If part) Assume now that the formula $\Phi$ is satisfiable.
Then, by Lemma \ref{lemma:intract}, there exists
a set of literals $R \subseteq \{x_1,\ldots,x_n\}$
such that $(D(\Phi),W(\Phi)_R)\models \neg(R\cup C)$.

Let $S$ be $R\cup C$. We next show that $(D'(\Phi),W'(\Phi)_S)\models \neg S$.
By semi-monotonicity, it follows that there exists at least
one extension ${\cal E}'$ of $(D'(\Phi),W'(\Phi)_S)$
such that ${\cal E}'\supseteq \neg(R\cup C)$.
Moreover,
since no default in
$D_4\cup D_5\cup D_6$
may belong to the set of generating defaults
of $(D'(\Phi),W'(\Phi)_S)$, it is the case that each
extension ${\cal E}'$ of $(D'(\Phi),W'(\Phi)_S)$
is such that ${\cal E}'\supseteq \neg(R\cup C)$.
Indeed, $f$ cannot be entailed by this theory since
no letter among
$c_1,\ldots,c_m$ appears in $W'(\Phi)_S$
and in the consequence of any default in $D'(\Phi)$ and,
moreover, $\neg l$ belongs to $W'(\Phi)_S$.
It can thus be concluded that $(D'(\Phi),W'(\Phi)_S)\models \neg S$.

Since, by virtue of the defaults in the set $D_6$,
there exists at least one extension $\cal E$ of the theory
$(D'(\Phi),W'(\Phi)_{S,\{\neg l\}})$
such that $f$ belongs to $\cal E$,
by Claim \ref{claimf}, that
$(D'(\Phi),W'(\Phi)_{S,\{\neg l\}})$ satisfies Condition 2
of Definition \ref{outlierD} is implied and, hence,
that $S$ is a strong outlier witness set for $L$.
Thus, $L=\{\neg l\}$ is a strong outlier set
in $\Delta'(\Phi)$.
This closes the proof. 
\end{proof}

We close the section by noticing that, as far as the complexity of the
\textit{Strong Outlier Recognition} problem
on propositional DF and general default theories is concerned,
the following result directly follows from
Theorem \ref{th:strout_cyc_rec}.

\begin{corollary}
\textit{Strong Outlier Recognition} for DF and general default theories is NP-hard.
\end{corollary}

\section{The tractabilty/intractability frontier for Outlier Existence}

In this section we characterize the complexity of the
outlier existence problem. We recall from Section \ref{sect:outlier}that
the Outlier Existence problem is: given a default theory
$\Delta=(D,W)$ and a positive integer $k$, is it true that there
exists and outlier $L$ in $\Delta$ such that $|L|\le k$?

We begin by considering the complexity of this problem on general (cyclic)
NU default theories.
\begin{theorem}\label{th:out_cyc_exist}
Over NU default theories, the problem \textit{Outlier Existence}
is $\NP$-complete.
\end{theorem}
\begin{proof}
The result immediately follows from \citep{ABP10},
where it is shown that \textit{Outlier Existence}
over NU theories is $\NP$-complete for any fixed constant $k$.
\end{proof}

Unluckily, restricting attention to strong outliers and quasi-acyclic
NU default theories does not allow to attain tractability on this problem,
as detailed by the following theorem.

\begin{theorem}\label{th:strong_outlier_existence}
{Strong Outlier Existence} for quasi-acyclic NU default theories is in NP-complete.
\end{theorem}
\begin{proof}
(Membership) Let $n$ be $|W|$.
It suffice to guess an outlier witness set $L\subset W$
of size not grater than $k\le n$ and
to check that
there exists a witness set $S\subset W$ of size at most $c$,
with $c$ the tightness of $\Delta$, such that
$(D,W_S)\entails\neg S$
and $(D,W_{S,L})\not\entails\neg S$.
The whole computation can be accomplished in nondeterministic
polynomial time by a Turing machine.

(Hardness)
The proof is by reduction to the well-known
NP-complete problem \textit{Hitting Set}:
given a collection $C=\{E_1,\ldots,E_m\}$ of subsets
of a finite set $V$ and a positive integer $k\le|V|$,
is there a subset $H\le V$, with $|H|\le k$, such that
$H$ contains at least one element from each subset in $C$?

Let $E_i=\{v_{i,1},\ldots,v_{i,n_i}\}$, with $1\le i\le m$.
The default theory $\Delta(C) = (D(C),W(C))$ is associated
with the collection $C$, where $W(C) = \neg V \cup \{ l, s \}$
and $D(C)=D_1\cup D_2\cup D_3$, with
\begin{eqnarray*}
D_1 & = & \left\{
\frac{l:\neg s}{\neg s} \right\},\\
D_2 & = & \left\{
\frac{:e_i}{e_i}, \frac{e_i:\neg s}{\neg s}
\mid i=1, \ldots, m \right\}, \mbox{ and}\\
D_3 & = & \left\{
\frac{:v_{i,j}}{v_{i,j}}, \frac{v_{i,j} : \neg e_i}{\neg e_i},
\mid i=1, \ldots, m; j = 1, \ldots, n_i \right\},\\
\end{eqnarray*}
where $l$ and $e_1,\ldots,e_m$ are new elements
distinct from those occurring in $V$.
Note that the theory $\Delta(C)$ is quasi-acyclic
and has tightness $c=1$.

It is shown next that $C$ has an hitting set of size
$k$ iff there exists a strong outlier set $L$ in $\Delta(C)$
having size $k+1$.

First of all, notice that $s$ is the unique literal in $W(C)$
that can be part of an outlier witness set, since
($i$)
for any literal $\ell$ in $W(C)$, $\ell$ does not appear in the conclusion
of a default rule, and ($ii$)
$s$ is the
only literal in $W(C)$ which appears negated in the conclusion of
a non-empty prerequisiste default.

Moreover, $S=\{s\}$ is indeed an outlier witness set, as
$(D(C),W(C)_{\{s\}})\entails\neg s$ by means of the default
in the set $D_1$.

(Only If part) Assume that $C$ has an hitting set $H$ of size $k$.
Consider the set $L=\{l\}\cup\{\neg v : v\in H\}$
having size $k+1$.
Now it is shown that $L$ is an outlier set.

Due the defaults in the set $D_2$ and since no default in $\Delta(C)$
has $s$ in its conclusion, if at least a letter $e_i$ belongs to an extension
of $(D(C),W(C)_{\{s\},L})$ then $\neg s$ belongs to the same extension,
Thus, $(D(C),W(C)_{\{s\},L})\not\entails \neg s$ iff there exists
an extension $\cal E$ such that no letter $e_i$ occurs in $\cal E$.

Consider now the defaults in the set $D_3$.
If at least a letter $v_{i,j}$ does not
occur in $W(C)_{\{s\},L}$, then there exists an extension which
contains $\neg e_i$ and does not contain $e_i$ ($1\le i\le m$).
Since, by construction, $H$ coincides with $L\setminus\{l\}$,
an extension $\cal E$ such that no letter $e_i$
occurs in $\cal E$ indeed exists. This proves that $L$
is an outlier set.

(If part)
Assume that there exists an outlier set $L\subseteq W(C)\setminus\{s\}$
of size $k+1$.

Then, it must be the case that $L\supseteq\{l\}$,
for otherwise $(D(C),W(C)_{\{s\},L})\entails\neg s$
by means of the default in $D_1$.

As already stated, it is the case that
$(D(C),W(C)_{\{s\},L})\not\entails \neg s$ if and only if there exists
an extension $\cal E$ such that no letter $e_i$ occurs in $\cal E$.

Such an extension exists provided that, for each $i\in\{1,\ldots,m\}$,
at least a letter $v_{i,j}$ occurs in $W(C)_{\{s\},L}$.
Since letters $e_i$ are associated with sets $E_i$,
it is the case that $L\setminus\{l\}$ encodes an hitting set $H$
having size $k$.
\end{proof}


\section{Tractable strong outlier enumeration algorithm}

Based on the above results, we are now ready to describe the algorithm
\textit{Strong Outlier Enumeration} which, for a fixed integer $k$, enumerates
in polynomial time all the strong outlier sets of size at most $k$
in a quasi-acyclic NU default theory.

\begin{figure}[t]
\begin{center}
\fbox{\begin{minipage}{0.95\textwidth}
\begin{algorithmic}[1]
\STATE \textbf{Input:} {$\Delta=(D,W)$ -- a NU default theory; $k$ -- a positive integer.}
\STATE \textbf{Output:} {$Out$ -- the set of all strong outlier sets $L$ in $\Delta$ s.t. $|L|\le k$.}
\medskip
\STATE let $C_1,\ldots,C_N$ the ordered SCCs in the atomic dependency graph of $\Delta$;
\STATE $Out = \emptyset$;
\FOR {$i=1$ \textbf{to} $N$}
\FORALL {$S\subset W$ s.t. $\lett{S}\subseteq C_i$}
\IF {$(\forall\ell\in S) (D,W_S)\models \neg\ell$}
\FORALL {$L\subseteq W_S$ s.t. $|L|\le k$ \textit{and} $\lett{L}\subseteq (C_1\cup\ldots\cup C_i)$}
\IF {$(\forall\ell\in S) (D,W_{S,L}) \not\models \neg\ell$}
\STATE $Out = Out \cup \{ L \}$;
\ENDIF
\ENDFOR
\ENDIF
\ENDFOR
\ENDFOR
\end{algorithmic}
\end{minipage}}
\end{center}
\caption{Algorithm \emph{Strong Outlier Enumeration}.}\label{fig:algorithm}
\end{figure}

The algorithm is detailed in Figure \ref{fig:algorithm}.
The SCCs $C_1,\ldots,C_N$ of the atomic
dependency graph of the theory are ordered such that there do not exist $C_i$ and $C_j$
with $i<j$ and two letters $l\in C_i$ and $q\in C_j$
such that there exists a path from $letter(q)$ to $letter(j)$.

The algorithm exploits Lemma \ref{sccsize} in order to
restrict the search space of witness sets to the subsets of the SCCs
of the atomic dependency graph (see steps 5-6),
exploits Lemma \ref{incremental} in order to consider as part of the candidate strong
outlier set $L$ only literals potentially influencing those included in the
current witness set $S$ (see step 8, where only SCCs preceding $C_i$
are taken into account),\footnote{As a further optimization,
only SCCs influencing $C_i$ should be taken into account in this step.}
and exploits Proposition \ref{th:nuentail} in order to solve the entailment
problem in $O(n^2)$ time (see steps 7 and 9).

As for the cost of the algorithm,
the number of strong outlier witness sets is $O(2^c(n/c))$,
$O(cn^2)$ is the cost of checking the first and the second condition
of the outlier definition,
while the number of of strong outlier sets is $O(n^k)$.
Summarizing, the cost of the algorithm is
$$O(2^c(n/c)(cn^2+n^kcn^2)) = O(2^c n^{k+3}).$$
Since $c$ and $k$ are fixed,
the algorithm enumerates the strong outliers
in polynomial time in the size of $(D,W)$.
For example, all the singleton strong outlier sets can be enumerated
in time $O(n^4)$.

\section{A technical discussion regarding the tractability/intractability frontier}
\label{sect:discussion}
Before closing the paper, we would like to offer some further comments
concerning the technical results presented above regarding the \textit{Outlier Recognition Problem}.
From Theorems \ref{th:out_acy_rec} and \ref{th:strout_cyc_rec}, it
is clear that neither quasi-acyclicity nor strongness alone are sufficient
to achieve tractability for this problem.
However, if both constraints are imposed together,
the complexity of the outlier recognition problem falls within the
tractability frontier, as shown in
Theorem \ref{th:strout_acy_rec}. In order to better understand the underlying rationale, we informally discuss next why the techniques exploited
in the proofs of Theorems \ref{th:out_acy_rec} and \ref{th:strout_cyc_rec}
fail in this case. By Lemma \ref{lemma:intract}, it follows that,
despite the difficulty to encode the
conjunction of a set of literals exploiting a NU theory, a CNF formula can still
be evaluated by means of condition $1$ of Definition \ref{outlierD}
applied to a quasi-acyclic NU theory.
In particular, in the construction of Lemma \ref{lemma:intract},
the letters $c_1, \ldots, c_m$ play the role of encoding
the truth values of the conjuncts $C_1, \ldots, C_m$ composing
the 3CNF formula $\Phi = f(X)$, while the subset $S(\Phi)$ of $X=\{x_1,\ldots,x_n\}$
encodes a truth value assignment to the variables in the set $X$
(specifically, $x_i$ is true iff it does not belong to $S(\Phi)$).
Hence, checking for
$(D(\Phi),W(\Phi)_{S(\Phi)}) \models \neg (S(\Phi) \cup \{c_1,\ldots,c_m\})$
is equivalent to verifying if $C_1\wedge\ldots\wedge C_m$ is true
under the truth value assignment encoded by $S(\Phi)$
(see Lemma \ref{lemma:intract} for the formal proof).

Now, in the presence of cyclicity, it is possible to constrain the witness set
$S$ to contain all the letters $c_1,\ldots,c_m$, by including them in the same
SCC of the atomic dependency graph and, thus, to resort to Lemma \ref{lemma:intract}
in order to prove NP-hardness (the reader is referred to Theorem \ref{th:strout_cyc_rec}
for the details).

In the quasi-acyclic case, it is still possible to constrain the witness set $S$
to contain all the letters $c_1,\ldots,c_m$, by exploiting condition $2$
of Definition \ref{outlierD} as done in the reduction of Theorem \ref{th:out_acy_rec}.
In particular,
the reduction introduces a dummy letter $c_0$ in the theory
such that $(D,W_{S,L})\not\models \neg c_0$.
As for the letters $c_1, \ldots, c_m$, their negation is entailed both by
$(D,W_S)$ and by $(D,W_{S,L})$. Thus,
in order to satisfy condition $2$ of Definition \ref{outlierD},
$c_0$ must be in $S$.
However, in order for $(D,W_S)\models \neg c_0$ to hold, it must be the case that
all letters $c_i$ ($1\le i\le m$) are in $S$.
In terms of the atomic dependency graph, this has been obtained
by adding an arc from each $c_i$ to $c_0$ together with an
arc from $l$ to $c_0$ (recall that $L=\{\neg l\}$ in the reduction),
without the need of introducing
cycles in the graph.

However,
if both quasi-acyclicity and outlier strongness are imposed, in order to exploit Lemma
\ref{lemma:intract}, it is needed to guarantee that for each $c_i$,
$(D,W_{S,L}) \not\models \neg c_i$ without the possibility of introducing
cycles in the atomic dependency graph.
Informally speaking, a possibility would be to add an arc
from $l$ to any $c_i$, so that $(D,W_{S,L})\models c_i$
( hence this theory does not entail $\neg c_i$), but the graph would become cyclic
(see the graph in Figure \ref{fig:out_acy_rec_graph}).
Moreover, removing the dummy letter $c_0$ from the reduction
in order to break the cycle would not help,
since in this case each singleton set $\{c_i\}$ would act as a potential
strong outlier witness set for $L=\{\neg l\}$.
It can be intuitively concluded that requiring quasi-acyclicity and strongness
simultaneously prevents the size of the minimal outlier witness to be unbounded
(this intuition is indeed formalized by Lemma \ref{sccsize})
and, hence, prevents the feasibility of reducing satisfiability
to the strong outlier recognition problem by exploiting the technique
depicted in Lemma \ref{lemma:intract}. Indeed, if the size of $S$
is bounded, CNFs having more than $|S|$ conjuncts cannot be
evaluated by means of condition $1$ of Definition \ref{outlierD}.

To conclude, we note that switching to cyclic theories, general outliers,
and NMU theories makes the cost of the \textit{Outlier Enumeration} algorithm exponential.
This is
justified by the results provided throughout the paper. 
Moreover, we point out the following observations:
\begin{itemize}
\item
For strong outliers, if the tightness $c$
is unbounded (on quasi-acyclic theories) the cost of the algorithm
depends exponentially on $c$
(see Theorem \ref{th:strout_cyc_rec} 
for the proof of
the intractability of the \textit{Strong Outlier Recognition} problem on quasi-acyclic NU theories).
\item
For general outliers, the algorithm still works provided
that in line 4 the set $S$ is constrained to be a subset
of $C_i\cup\ldots\cup C_N$, instead of a subset of $C_i$, and
that in line 6 the set $L$ is constrained to be a subset
of $C_1\cup\ldots\cup C_w$, where
$w=\max\{ j : l\in S \mbox{ and } letter(l) \in C_j \}$,
instead of a subset of $C_1\cup\ldots\cup C_i$ (note that $w\ge i$).
However, in this case, the algorithm depends exponentially on $n$, even if
it takes advantage of the structural property provided by the
Incremental Lemma (see Theorem \ref{th:out_acy_rec} 
for the proof of
the intractability of the \textit{Outlier Recognition} problem on NU theories).
To break the exponential dependency on $n$, the size of the
outlier witness set $S$ should be constrained
to be within a certain fixed threshold $h$, that is $|S|\le h$.
\item
The algorithm can be applied also to NMU theories, but in this
case the cost of steps 7 and 9 depends exponentially on the size of the theory
(the reader is referred to Theorem \ref{th:nmu_entail} 
for the proof of the co-NP-completeness of the entailment problem
for propositional (quasi-acyclic) NMU theories).
\end{itemize}

\section{Conclusions}
\label{sect:concl}

Traditional approaches model the normal behavior of individuals by performing
some statistical analysis  on the given data set and, then, singling out
those individuals whose behavior or characteristics significantly deviate from
normal ones.
On the other hand, it is supposedly quite interesting to
exploit {domain knowledge} in order to guide the search for anomalous
observations.
%



The outlier detection technique investigated here
is based on the definition introduced in \citep{ABP08},
which is an unsupervised one, in that no examples of normality/abnormality
are required.
This technique can be applied to databases
including observations to be examined, databases
provided by the organization which is interested in learning exceptional
individuals.
In order to exploit domain knowledge, databases are to be coupled
with a knowledgebase composed of default rules and observations.
The default rules can be supplied directly by the knowledge engineer or
they can be a product of a rule learning module.
E.g., the rule learning step can be based on theory and algorithms
developed for learning default rules \citep{DuvalN99}
and/or for metaquerying \citep{SOMZ96,BGI03,AngiulliBIP03}\footnote{Metaquerying
is a formal tool for learning rules that involve several
relations in the database and a metaquery directs the search by providing a
(partially specified) pattern of the rules of interest.}.

%


As a major contribution of this paper, we have analyzed
the tractability border associated with outlier
detection problems in default logics.

Overall, the results and arguments reported in this paper
indicate that outlier recognition, even in its strong version, remains challenging
and difficult on general default theories.
The tractability results we have provided nonetheless indicate that there.
are significant cases which can
be efficiently implemented.

\bibliographystyle{elsarticle-harv}

\end{document}

%% file: macros.tex
\newtheorem{definition}{Definition}[section]
\newtheorem{theorem}[definition]{Theorem}
\newtheorem{lemma}[definition]{Lemma}
\newtheorem{corollary}[definition]{Corollary}
\newtheorem{proposition}[definition]{Proposition}

\newenvironment{proof}{\noindent {\bf Proof:}}{\qed}

\newtheorem{claim}{Claim}
\newenvironment{claimproof}[1]{\noindent {\bf Proof of Claim \ref{#1}:}}{\qed}

\def\squarebox#1{\hbox to #1{\hfill\vbox to #1{\vfill}}}

\newcommand{\nop}[1]{}

\newcommand{\newc}{\newcommand}

\newc{\comment}[1]{}

\newc{\ael}[1]{\mbox{\mbox{\bf{L}}}#1}
\newc{\card}[1]{\mbox{$|\mbox{#1}|$}}
\newc{\dl}{\mbox{$(D,W)$ \vspace{0.5em}}}
\newc{\dr}{\mbox{$\alpha : \beta / \gamma$}}
\newc{\edge}{\mbox{$\longrightarrow$}}
\newc{\eg}{\mbox{e.g., }}
\newc{\entails}{\mbox{$\models$}}
\newc{\es}{\mbox{$E^*$}}
\newc{\ext}{\mbox{$\th{E}$}}
\newc{\false}{\mbox{\bf false}}
\newc{\ie}{\mbox{i. e.}}
\newc{\impl}{\mbox{$ \hspace{0.1em} \Rightarrow \hspace{0.1em}$}}
\newc{\inter}{\mbox{$\bigcap$ }}
\newc{\lang}{\mbox{$\cal L$}}
\newc{\lc}[1]{\mbox{$#1^*$}}
\newc{\lett}[1]{\mbox{$letter(#1)$}}
\newc{\limp}{\mbox{$\longrightarrow$}}
\newc{\limm}{\mbox{$\longleftrightarrow$}}
\newc{\lpimp}{\mbox{$\vspace{0.5em}\longleftarrow\vspace{0.5em}$}}
\newc{\lpneg}{\mbox{{\it not} \vspace{0.5em}}}
\newc{\lpor}{\mbox{\vspace{0.5em} $\mid$ \vspace{0.5em}}}
\newc{\lit}{\mbox{$\lang$}}

\newc{\liter}[1]{\mbox{$liter(#1)$}}
\newc{\nat}{\mbox{$\cal N$}}
\newc{\negg}{\mbox{$\sim$}}
\newc{\notsubseteq}{\mbox{$\subseteq \hspace{-0.8em} / \hspace{0.15em} $}}
\newc{\ncomp}{\mbox{$\comp \hspace{-0.5em} \setminus$}}
\newc{\nmodels}{\mbox{$\models \hspace{-0.8em} \setminus$}}
\newc{\npr}{\mbox{$\prov \hspace{-0.5em} / \hspace{0.25em} $}}
\newc{\parl}{\mbox{$\parallel$}}
\newc{\pnp}{\mbox{$\mbox{$P^{NP[O(\log n)]}$}$}}
\newc{\pptc}{\mbox{$\Pi^p_2$-complete}}
\newc{\pr}{\mbox{{\em Proof: }}}
\newc{\prov}{\mbox{$\vdash$}}
\newc{\pvar}{\mbox{\bf \cal V}}
\newc{\rel}{\mbox{$\cal R$}}
\newc{\res}[2]{\mbox{$res(#1,#2)$}}
\newc{\refi}[1]{\mbox{(\ref{#1})}}
\newc{\st}{\mbox{$\mid$}}
\newc{\tag}[1]{\mbox{$#1^\prime$}}
\newc{\tagg}[1]{\mbox{$#1^{\prime \prime}$}}
\newc{\true}{\mbox{\bf true}}
\newc{\union}{\mbox{$\bigcup$}}
\newc{\wlg}{\mbox{w.\ l.\ g.\ }}
\newc{\wrt}{\mbox{w.\ r.\ t.\ }}
\newc{\ws}{\mbox{$W \union S$}}

%% file: main-tcs2003-arxiv.bbl
\begin{thebibliography}{24}
\expandafter\ifx\csname natexlab\endcsname\relax\def\natexlab#1{#1}\fi
\providecommand{\url}[1]{\texttt{#1}}
\providecommand{\href}[2]{#2}
\providecommand{\path}[1]{#1}
\providecommand{\DOIprefix}{doi:}
\providecommand{\ArXivprefix}{arXiv:}
\providecommand{\URLprefix}{URL: }
\providecommand{\Pubmedprefix}{pmid:}
\providecommand{\doi}[1]{\href{http://dx.doi.org/#1}{\path{#1}}}
\providecommand{\Pubmed}[1]{\href{pmid:#1}{\path{#1}}}
\providecommand{\bibinfo}[2]{#2}
\ifx\xfnm\relax \def\xfnm[#1]{\unskip,\space#1}\fi
\bibitem[{Abiteboul et~al.(1995)Abiteboul, Hull and Vianu}]{AbiteboulHV95}
\bibinfo{author}{Abiteboul, S.}, \bibinfo{author}{Hull, R.},
  \bibinfo{author}{Vianu, V.}, \bibinfo{year}{1995}.
\newblock \bibinfo{title}{Foundations of Databases}.
\newblock \bibinfo{publisher}{Addison-Wesley}.
\bibitem[{Angiulli et~al.(2003)Angiulli, Ben-Eliyahu-Zohary, Ianni and
  Palopoli}]{AngiulliBIP03}
\bibinfo{author}{Angiulli, F.}, \bibinfo{author}{Ben-Eliyahu-Zohary, R.},
  \bibinfo{author}{Ianni, G.}, \bibinfo{author}{Palopoli, L.},
  \bibinfo{year}{2003}.
\newblock \bibinfo{title}{Computational properties of metaquerying problems}.
\newblock \bibinfo{journal}{ACM Trans. Comput. Log.} \bibinfo{volume}{4},
  \bibinfo{pages}{149--180}.
\bibitem[{Angiulli et~al.(2008)Angiulli, Zohary and Palopoli}]{ABP08}
\bibinfo{author}{Angiulli, F.}, \bibinfo{author}{Zohary, R.B.E.},
  \bibinfo{author}{Palopoli, L.}, \bibinfo{year}{2008}.
\newblock \bibinfo{title}{Outlier detection using default reasoning}.
\newblock \bibinfo{journal}{Artificial Intelligence} \bibinfo{volume}{172},
  \bibinfo{pages}{1837--1872}.
\bibitem[{Angiulli et~al.(2010)Angiulli, Zohary and Palopoli}]{ABP10}
\bibinfo{author}{Angiulli, F.}, \bibinfo{author}{Zohary, R.B.E.},
  \bibinfo{author}{Palopoli, L.}, \bibinfo{year}{2010}.
\newblock \bibinfo{title}{Outlier detection for simple default theories}.
\newblock \bibinfo{journal}{Artificial Intelligence} \bibinfo{volume}{174},
  \bibinfo{pages}{1247--1253}.
\bibitem[{Ben-Eliyahu and Dechter(1994)}]{BeDe94}
\bibinfo{author}{Ben-Eliyahu, R.}, \bibinfo{author}{Dechter, R.},
  \bibinfo{year}{1994}.
\newblock \bibinfo{title}{Propositional semantics for disjunctive logic
  programs}.
\newblock \bibinfo{journal}{Annals of Mathematics and Artificial Intelligence}
  \bibinfo{volume}{12}, \bibinfo{pages}{53--87}.
\bibitem[{Ben-Eliyahu-Zohary et~al.(2003)Ben-Eliyahu-Zohary, Gudes and
  Ianni}]{BGI03}
\bibinfo{author}{Ben-Eliyahu-Zohary, R.}, \bibinfo{author}{Gudes, E.},
  \bibinfo{author}{Ianni, G.}, \bibinfo{year}{2003}.
\newblock \bibinfo{title}{Metaqueries: Semantics, complexity, and efficient
  algorithms}.
\newblock \bibinfo{journal}{Artif. Intell.} \bibinfo{volume}{149},
  \bibinfo{pages}{61--87}.
\bibitem[{Cadoli et~al.(1997)Cadoli, Eiter and Gottlob}]{CadoliEG97}
\bibinfo{author}{Cadoli, M.}, \bibinfo{author}{Eiter, T.},
  \bibinfo{author}{Gottlob, G.}, \bibinfo{year}{1997}.
\newblock \bibinfo{title}{Default logic as a query language}.
\newblock \bibinfo{journal}{IEEE Trans. Knowl. Data Eng.} \bibinfo{volume}{9},
  \bibinfo{pages}{448--463}.
\bibitem[{Dantsin et~al.(2001)Dantsin, Eiter, Gottlob and
  Voronkov}]{DantsinEGV01}
\bibinfo{author}{Dantsin, E.}, \bibinfo{author}{Eiter, T.},
  \bibinfo{author}{Gottlob, G.}, \bibinfo{author}{Voronkov, A.},
  \bibinfo{year}{2001}.
\newblock \bibinfo{title}{Complexity and expressive power of logic
  programming}.
\newblock \bibinfo{journal}{ACM Comput. Surv.} \bibinfo{volume}{33},
  \bibinfo{pages}{374--425}.
\bibitem[{Duval and Nicolas(1999)}]{DuvalN99}
\bibinfo{author}{Duval, B.}, \bibinfo{author}{Nicolas, P.},
  \bibinfo{year}{1999}.
\newblock \bibinfo{title}{Learning default theories}, in:
  \bibinfo{booktitle}{ESCQARU}, pp. \bibinfo{pages}{148--159}.
\bibitem[{Ebbinghaus and Flum(1995)}]{Ebbin95}
\bibinfo{author}{Ebbinghaus, H.D.}, \bibinfo{author}{Flum, J.},
  \bibinfo{year}{1995}.
\newblock \bibinfo{title}{Finite model theory}.
\newblock Perspectives in Mathematical Logic, \bibinfo{publisher}{Springer}.
\bibitem[{Gottlob(1992)}]{Gottlob92}
\bibinfo{author}{Gottlob, G.}, \bibinfo{year}{1992}.
\newblock \bibinfo{title}{Complexity results for nonmonotonic logics}.
\newblock \bibinfo{journal}{J. Log. Comput.} \bibinfo{volume}{2},
  \bibinfo{pages}{397--425}.
\bibitem[{Johnson(1990)}]{Johnson91}
\bibinfo{author}{Johnson, D.S.}, \bibinfo{year}{1990}.
\newblock \bibinfo{title}{Handbook of theoretical computer science (vol. a)},
  \bibinfo{publisher}{MIT Press}, \bibinfo{address}{Cambridge, MA, USA}.
  chapter \bibinfo{chapter}{A catalog of complexity classes}, pp.
  \bibinfo{pages}{67--161}.
\bibitem[{Kautz and Selman(1991)}]{KaSe91}
\bibinfo{author}{Kautz, H.A.}, \bibinfo{author}{Selman, B.},
  \bibinfo{year}{1991}.
\newblock \bibinfo{title}{Hard problems for simple default logics}.
\newblock \bibinfo{journal}{Artificial Intelligence} \bibinfo{volume}{49},
  \bibinfo{pages}{243--279}.
\bibitem[{Minker(2000)}]{Min00}
\bibinfo{editor}{Minker, J.} (Ed.), \bibinfo{year}{2000}.
\newblock \bibinfo{title}{Logic-Based Artificial Intelligence}.
\newblock \bibinfo{publisher}{Kluwer Academic Publisher}.
\bibitem[{Neapolitan and Xia(2012)}]{NeaJ12}
\bibinfo{author}{Neapolitan, R.}, \bibinfo{author}{Xia, J.},
  \bibinfo{year}{2012}.
\newblock \bibinfo{title}{Contemporary Artificial Intelligence}.
\newblock \bibinfo{publisher}{Chapman \& Hall/CRC}.
\bibitem[{Papadimitriou(1994)}]{pap94}
\bibinfo{author}{Papadimitriou, C.H.}, \bibinfo{year}{1994}.
\newblock \bibinfo{title}{Computatational Complexity}.
\newblock \bibinfo{publisher}{Addison-Wesley, Reading, Mass.}
\bibitem[{Poole et~al.(1998)Poole, Mackworth and Goebel}]{PooMG98}
\bibinfo{author}{Poole, D.}, \bibinfo{author}{Mackworth, A.K.},
  \bibinfo{author}{Goebel, R.}, \bibinfo{year}{1998}.
\newblock \bibinfo{title}{Computational Intelligence - a logical approach}.
\newblock \bibinfo{publisher}{Oxford University Press}, \bibinfo{address}{New
  York}.
\bibitem[{Reiter(1980)}]{Rei80}
\bibinfo{author}{Reiter, R.}, \bibinfo{year}{1980}.
\newblock \bibinfo{title}{A logic for default reasoning}.
\newblock \bibinfo{journal}{Artificial Intelligence} \bibinfo{volume}{13},
  \bibinfo{pages}{81--132}.
\bibitem[{Russell and Norvig(2003)}]{RussellN03}
\bibinfo{author}{Russell, S.J.}, \bibinfo{author}{Norvig, P.},
  \bibinfo{year}{2003}.
\newblock \bibinfo{title}{Artificial Intelligence: A Modern Approach}.
\newblock \bibinfo{edition}{2nd} ed., \bibinfo{publisher}{Prentice Hall},
  \bibinfo{address}{Upper Saddle River, New Jersey}.
\bibitem[{Shen et~al.(1996)Shen, Ong, Mitbander and Zaniolo}]{SOMZ96}
\bibinfo{author}{Shen, W.}, \bibinfo{author}{Ong, K.},
  \bibinfo{author}{Mitbander, B.}, \bibinfo{author}{Zaniolo, C.},
  \bibinfo{year}{1996}.
\newblock \bibinfo{title}{Metaqueries for data mining}, in:
  \bibinfo{editor}{Fayyad, U.M.}, \bibinfo{editor}{Piatetsky-Shapiro, G.},
  \bibinfo{editor}{Smyth, P.}, \bibinfo{editor}{Uthurusamy, R.} (Eds.),
  \bibinfo{booktitle}{Advances in knowledge discovery and data mining}.
  \bibinfo{publisher}{AAAI Press / the MIT Press}, pp.
  \bibinfo{pages}{375--398}.
\bibitem[{Stillman(1992)}]{Stillman92}
\bibinfo{author}{Stillman, J.}, \bibinfo{year}{1992}.
\newblock \bibinfo{title}{The complexity of propositional default logics}, in:
  \bibinfo{booktitle}{AAAI}, pp. \bibinfo{pages}{794--799}.
\bibitem[{Winston(1998)}]{Win84}
\bibinfo{author}{Winston, P.H.}, \bibinfo{year}{1998}.
\newblock \bibinfo{title}{Artificial Intelligence}.
\newblock \bibinfo{publisher}{Addison-Wesley}, \bibinfo{address}{Reading,
  Massachusetts}.
\bibitem[{Zhang and Marek(1990)}]{ZhMa90}
\bibinfo{author}{Zhang, A.}, \bibinfo{author}{Marek, W.}, \bibinfo{year}{1990}.
\newblock \bibinfo{title}{On the classification and existence of structures in
  default logic}.
\newblock \bibinfo{journal}{Fundamenta Informaticae} \bibinfo{volume}{13},
  \bibinfo{pages}{485--499}.
\bibitem[{Zohary(2002)}]{Ben02}
\bibinfo{author}{Zohary, R.B.E.}, \bibinfo{year}{2002}.
\newblock \bibinfo{title}{Yet some more complexity results for default logic}.
\newblock \bibinfo{journal}{Artificial Intelligence} \bibinfo{volume}{139},
  \bibinfo{pages}{1--20}.

\end{thebibliography}
